\documentclass[preprint,12pt]{elsarticle}

\usepackage{amsfonts,amsmath,amssymb,amsthm}
\usepackage{physics,mathtools,bm,mathrsfs}
\usepackage{xcolor,enumerate,autobreak}
\allowdisplaybreaks[4]
\usepackage{wrapfig,subfigure,multirow}
\usepackage{hyperref,prettyref}
\usepackage{natbib}
\usepackage{tikz}
\usetikzlibrary{decorations.markings}
\usetikzlibrary{arrows}

\theoremstyle{plain}
\newtheorem{theorem}{Theorem}[section]
\newtheorem{lemma}[theorem]{Lemma}
\newtheorem{corollary}[theorem]{Corollary}

\theoremstyle{definition}
\newtheorem{proposition}[theorem]{Proposition}
\newtheorem{definition}[theorem]{Definition}
\newtheorem{remark}[theorem]{Remark}
\newtheorem{example}{Example}[section]
\newtheorem{assumption}{Assumption}

\newrefformat{eq}{\textup{(\ref{#1})}}
\newrefformat{lem}{Lemma~\textup{\ref{#1}}}
\newrefformat{thm}{Theorem~\textup{\ref{#1}}}
\newrefformat{cor}{Corollary~\textup{\ref{#1}}}
\newrefformat{prop}{Proposition~\textup{\ref{#1}}}
\newrefformat{assu}{Assumption~\textup{\ref{#1}}}
\newrefformat{remark}{Remark~\textup{\ref{#1}}}
\newrefformat{example}{Example~\textup{\ref{#1}}}
\newrefformat{def}{Definition~\textup{\ref{#1}}}
\newrefformat{sec}{Section~\textup{\ref{#1}}}
\newrefformat{subsec}{Subsection~\textup{\ref{#1}}}
\newrefformat{subsubsec}{Subsection~\textup{\ref{#1}}}

\newcommand{\ep}{\varepsilon}
\newcommand{\R}{\mathbb{R}}
\newcommand{\E}{\mathbb{E}}
\newcommand{\K}{\mathbb{K}}

\newcommand{\caH}{\mathcal{H}}

\newcommand{\caN}{\mathcal{N}}

\newcommand{\caR}{\mathcal{R}}

\newcommand{\caX}{\mathcal{X}}
\newcommand{\caY}{\mathcal{Y}}

\newcommand{\bbS}{\mathbb{S}}
\newcommand{\bbZ}{\mathbb{Z}}
\newcommand{\bbN}{\mathbb{N}}
\newcommand{\bbP}{\mathbb{P}}
\newcommand{\bbC}{\mathbb{C}}
\newcommand{\bbT}{\mathbb{T}}
\newcommand{\bbB}{\mathbb{B}}

\DeclareMathOperator{\Ran}{Ran}
\DeclareMathOperator{\Ker}{Ker}

\newcommand{\mr}{\mathrm}
\providecommand{\ang}[1]{\left\langle{#1}\right\rangle}
\providecommand{\cref}{\prettyref}

\newcommand{\reg}{\varphi_\lambda}
\newcommand{\rem}{\psi_\lambda}

\long\def\rev#1{#1}

\newcommand{\hf}{\frac{1}{2}}
\newcommand{\fstar}{f^*}
\newcommand{\flam}{f^*_{\lambda}}
\newcommand{\gtl}{\tilde{g}_X}

\journal{Applied and Computational Harmonic Analysis}

\begin{document}

\begin{frontmatter}

\title{Generalization Error Curves for Analytic Spectral Algorithms under Power-law Decay}

\author[CSS]{Yicheng Li}
\ead{liyc22@mails.tsinghua.edu.cn}

\author[DMS]{Weiye Gan}
\ead{gwy22@mails.tsinghua.edu.cn}

\author[DMS,Yau]{Zuoqiang Shi}
\ead{zqshi@tsinghua.edu.cn}

\author[CSS]{Qian Lin\corref{cor1}}
\ead{qianlin@tsinghua.edu.cn}

\affiliation[CSS]{organization={
  Department of Statistics and Data Science, Tsinghua University},%
            addressline={Haidian District},
            city={Beijing},
            postcode={100084},
            state={Beijing},
            country={China}}

\affiliation[DMS]{organization={Department of Mathematical Sciences, Tsinghua University},%
            addressline={Haidian District},
            city={Beijing},
            postcode={100084},
            state={Beijing},
            country={China}}

\affiliation[Yau]{organization={Yau Mathematical Sciences Center, Tsinghua University},%
            addressline={Haidian District},
            city={Beijing},
            postcode={100084},
            state={Beijing},
            country={China}}

\cortext[cor1]{Corresponding author.}

\begin{abstract}

The generalization error curve of a kernel regression method concerns the exact order of the generalization error under various source conditions, noise levels, and choices of the regularization parameter, rather than only the minimax rate.
In this work, under mild assumptions, we rigorously characterize the generalization error curves of kernel gradient descent and, more generally, of a large class of analytic spectral algorithms in kernel regression.
Consequently, we sharpen the near-inconsistency result for kernel interpolation and clarify the saturation effects of kernel regression algorithms with higher qualification.
Motivated in part by neural tangent kernel theory, these results greatly improve our understanding of the generalization behavior of wide neural networks.
A novel technical contribution, the analytic functional argument, may also be of independent interest.

 \end{abstract}


\end{frontmatter}

\section{Introduction}

The neural tangent kernel (NTK)
theory~\citep{jacot2018_NeuralTangent}, which shows
that kernel gradient descent well approximates over-parametrized neural networks trained by gradient descent
~\citep{jacot2018_NeuralTangent,allen-zhu2019_ConvergenceTheory,lee2019_WideNeurala},
provides a natural surrogate for understanding the generalization behavior of neural networks in certain circumstances.
This surrogate has led to a recent renaissance in the study of kernel methods.
For example, one may ask whether overfitting could harm generalization~\citep{bartlett2020_BenignOverfitting},
how the smoothness of the underlying regression function would affect the generalization error~\citep{li2023_SaturationEffect},
or how one can determine a lower bound on the generalization error for a specific function.
All these problems can be answered by the \textit{generalization error curve}
which aims to determine the exact generalization error of a certain kernel regression method with respect to
the kernel, the regression function, the noise level and the choice of the regularization parameter.
It is clear that such a generalization error curve would provide a comprehensive picture of the generalization ability of the corresponding kernel regression method
~\citep{bordelon2020_SpectrumDependent,cui2021_GeneralizationError,li2023_AsymptoticLearning}.

Although there has been extensive work on the generalization errors of kernel regression,
most of them focused on the optimal rate of convergence under the minimax framework.
For example, \citet{caponnetto2007_OptimalRates} showed that, with a proper choice of the regularization parameter,
kernel ridge regression (KRR) can achieve the minimax optimal rate of convergence.
Being a special case, KRR falls into a large class of kernel methods often referred to as \textit{spectral algorithms}~\citep{rosasco2005_SpectralMethods,gerfo2008_SpectralAlgorithms}.
For general spectral algorithms, subsequent works (e.g., \citet{blanchard2018_OptimalRates,lin2018_OptimalRatesa}) proved similar optimality results.
We refer the reader to \cref{subsec:related_works} for more details.
However, these works are not sufficient to answer the aforementioned problems motivated by recent studies of neural networks, since they only considered the method-dependent upper bound and the method-independent minimax lower bound of the generalization error.
In addition, most of them focused mainly on the rate of convergence and ignored the constant factors.%
Therefore, these traditional results are not enough to provide a comprehensive picture of the generalization error of kernel methods.
 
Going beyond the traditional results, several recent works have attempted to describe the generalization error curve of kernel ridge regression (KRR).
With some heuristic arguments,
\citet{bordelon2020_SpectrumDependent,cui2021_GeneralizationError} derived the generalization error curve of KRR under certain restrictive assumptions.
Under mild assumptions, \citet{li2023_AsymptoticLearning} first rigorously characterized the generalization error curve of KRR in terms of asymptotic convergence rate.
They showed an exact U-shaped bias-variance trade-off for the generalization error of KRR with respect to the choice of the regularization parameter.
Since neural networks are often trained by gradient descent,
it is of great interest to further study the generalization error curves of kernel gradient descent.
To the best of our knowledge, little attention has been paid to this aspect.

In this paper, we study the generalization error curves of a large class of analytic spectral algorithms, including the kernel gradient method.
To be precise, let $\rho$ be a probability distribution on $\caX \times \R$ and $\fstar$ be the unknown regression function~\citep{andreaschristmann2008_SupportVector},
namely the conditional expectation $\fstar(x) = \E_{\rho}(y|x)$.
Given $n$ i.i.d.\ samples $\{(x_i,y_i)\}_{i=1}^n$, let $\hat{f}_\lambda$ be the estimator given by a spectral algorithm with regularization parameter $\lambda = \lambda(n) > 0$.

Then, our result shows that for $\lambda$ in a reasonable range,
\begin{align*}
    \E \left[ \norm{\hat{f}_\lambda - \fstar}_{L^2}^2 ~\Big|~ X \right]
      = (1 + o_{\bbP}(1))\left( \caR_{\varphi}^2(\lambda;\fstar) + \frac{\sigma^2}{n} \caN_{2,\varphi}(\lambda) \right),
\end{align*}
where the conditional expectation is taken with respect to the training sample $X = (x_1,\dots,x_n)$,
$\sigma^2$ is the variance of the noise,
$\caR_{\varphi}^2(\lambda;\fstar)$ and $\caN_{2,\varphi}(\lambda)$ are two deterministic quantities (see \cref{eq:BiasMainTerm} and \cref{eq:PhiEffDim}) corresponding to
the bias and the variance respectively, and $\varphi$ is the filter function defining the spectral algorithm (see \cref{eq:SA}).
Moreover, if $\lambda$ does not lie in the reasonable range, we also show that the generalization error is bounded below by a nearly constant quantity.
The assumptions made in this paper are also mild and are satisfied for many RKHSs and spectral algorithms.
We refer to \cref{sec:main-results} for a complete statement for our main result.

With the exact $1+o_{\bbP}(1)$ form, our result characterizes \textit{exactly and completely} the generalization error for a large class of analytic spectral algorithms.
In particular, it shows a clear U-shaped bias-variance trade-off curve, in which the bias decreases while the variance increases as the regularization strength $\lambda$ decreases,
where the optimal point corresponds to the minimax optimal rate of convergence.
The result also shows that when the regularization is too weak, an overfitted estimator cannot generalize, emphasizing the necessity of regularization.
Moreover, our result also reveals a high-order saturation effect for some specific spectral algorithms.
Our result greatly improves the understanding of the generalization behavior of spectral algorithms and wide neural networks.%

Finally, the novel application of the ``analytic functional argument'' in deriving sharp estimates for spectral algorithms might be of independent interest and worthy of further investigation.

\subsection{Related works}
\label{subsec:related_works}
\paragraph{Optimality of kernel methods}
There is a large body of work studying the optimal rates of kernel ridge regression and spectral algorithms.

The classical work~\citep{caponnetto2007_OptimalRates} proved the minimax optimality of KRR when the regression lies in the RKHS;
subsequent works~\citep{steinwart2009_OptimalRates,fischer2020_SobolevNorm,zhang2023_OptimalityMisspecified}
further extend the result to the misspecified cases when the regression function does not lie in the RKHS\@.
\citet{zhang2005_BoostingEarly} and \citet{yao2007_EarlyStopping} considered the kernel gradient method and proved consistency and fast rates of convergence respectively.
General spectral algorithms were first introduced by \citet{bauer2007_RegularizationAlgorithms}
and then studied extensively in follow-up works~\citep{gerfo2008_SpectralAlgorithms,rosasco2005_SpectralMethods},
but the eigenvalue decay (see \cref{assu:EDR}) of the kernel was not considered so the rates are not optimal.
Under certain restrictive conditions, \citet{caponnetto2006_OptimalRates} proved the minimax optimality of spectral algorithms.
More recently, a sequence of works further extended the result to more general cases (e.g., \citet{blanchard2018_OptimalRates,lin2018_OptimalRatesa}).
In addition, the very recent work~\citep{zhang2023_OptimalityMisspecifieda} showed the optimality for the misspecified cases even when the regression function is unbounded.
We also refer to Table 1 in \citet{zhang2023_OptimalityMisspecifieda} for a summary of the results.
However, as discussed above, these results only focused on the upper bounds and are not enough to provide the exact generalization error curve.

\paragraph{Recent advances in kernel ridge regression}
Focusing particularly on KRR, a recent line of work provides further results on its generalization.
Some works~\citep{rakhlin2018_ConsistencyInterpolation,buchholz2022_KernelInterpolation,beaglehole2022_KernelRidgeless,li2023_KernelInterpolation}
studied kernel ridgeless regression, which is the limiting case of KRR as the regularization goes to zero, and proved that it cannot generalize.
Using a restrictive Gaussian design assumption together with some non-rigorous arguments,
\citet{bordelon2020_SpectrumDependent,cui2021_GeneralizationError} derived the generalization error curve of KRR
and \citet{mallinar2022_BenignTempered} studied further the interpolation regime.
For rigorous results, \citet{li2023_AsymptoticLearning} proved the generalization error curve with asymptotic rates in the form of $\Theta_{\bbP}(n^{-r})$,
but the hidden constant factors are not tracked.

\paragraph{Kernel regression in the high-dimensional limit}
There is also a line of works studying the generalization of kernel regression in the high-dimensional limit when the dimension of the input space $d$ diverges with $n$.
For example, in the high-dimensional setting, \citet{liang2020_JustInterpolate} showed that kernel interpolation can generalize;
\citet{ghorbani2020_LinearizedTwolayers,ghosh2021_ThreeStages,liu2021_KernelRegression,lu2023_OptimalRate} studied the generalization of kernel ridge regression and kernel gradient method.
However, we emphasize that their results can be substantially different from ours because the setting is different.
Moreover, the high dimensionality in their setting actually makes the problem easier since the kernel can effectively be linearized and the well-established random matrix theory can be applied, which is not the case in our setting.

\section{Preliminaries}

\subsection{Reproducing kernel Hilbert space}

Let $\caX$ be a compact metric input space and $\caY \subseteq \R$ be the output space.
Let $\rho$ be the unknown probability measure supported on $\caX \times \caY$ and $\mu$ be the marginal probability measure of $\rho$ on $\caX$.
Denote by $L^2 = L^2(\caX,\dd \mu)$ the space of (complex-valued) square-integrable functions on $\caX$.
Assume that $\E_{(x,y)\sim \rho}(y^2) < \infty$, and let the conditional expectation
\begin{align}
  \fstar(x) = \E_{\rho}(y \mid x) = \int_{\caY} y \dd \rho(y|x) \in L^2
\end{align}
be the \textit{regression function}.
We fix a continuous positive definite\footnote{
  We consider complex-valued spaces here since the analytic functional argument later will be based on complex analysis.
  Then, $k$ is conjugate symmetric, $k(x,y) = \overline{k(y,x)}$, and positive definite in the sense that
  $\sum_{i,j = 1}^n \overline{c_i} c_j k(x_i,x_j) \geq 0$ for any $n \geq 1$, $x_1,\dots,x_n \in \caX$ and $c_1,\dots,c_n \in \bbC$.
} kernel $k : \caX \times \caX \to \bbC$ over $\caX$
and let $\caH$ be the (complex) separable reproducing kernel Hilbert space (RKHS) associated with $k$.
Note that we adopt the convention that inner product is linear in the second component and conjugate linear in the first component,
that is, $\ang{f,g}_{L^2} = \int_{\caX} \overline{f(x)} g(x) \dd \mu(x)$ and $\ang{k(x,\cdot),f}_{\caH} = f(x)$ for $f \in \caH$.
Since $\caX$ is compact and $k$ is continuous, we have $\sup_{x \in \caX} k(x,x) \leq \kappa^2 < \infty$.
Consequently, we have the natural inclusion $S_k : \caH \to L^2$ which is Hilbert-Schmidt~\citep{andreaschristmann2008_SupportVector,steinwart2012_MercerTheorem}.
Denote by $S_k^* : L^2 \to \caH$ the adjoint operator of $S_k$.
Then, $T = S_k S_k^* : L^2 \to L^2$ defines an integral operator
\begin{align}
  \label{eq:IntegralOperator}
  (Tf)(x)  = \int_{\mathcal{X}} \overline{k(x,x')} f(x') \dd \mu(x').
\end{align}
Moreover, it is well known~\citep{caponnetto2007_OptimalRates,steinwart2012_MercerTheorem} that $T$ is a self-adjoint, positive, trace-class operator with trace norm $\norm{T}_1 \leq \kappa^2$.
Focusing on the infinite-dimensional case where $T$ is not of finite rank, we have the spectral decomposition
\begin{align}
  \label{eq:T_decomp}
  T = \sum_{m =1}^\infty \mu_m P_{V_m},
\end{align}
where $(\mu_m)_{m \geq 1} $ is the decreasing sequence of the \textit{distinct} positive eigenvalues of $T$ and $P_{V_m}$ is the projection onto the eigenspace $V_m$ associated with $\mu_m$.
Denote by $d_m = \dim V_m$ the multiplicity of $\mu_m$.
Let us further choose an orthonormal basis $\{e_{m,l}\}_{l = 1}^{d_m}$ for each $V_m$, where each $e_{m,l}$ is the continuous representative of the corresponding $\mu$-equivalence class.
Then, $\{ e_{m,l}\}$ forms an orthonormal basis of $\overline{\Ran S_k} = (\Ker S_k^*)^{\perp} \subseteq L^2$ and
$\{\mu_m^{1/2} e_{m,l}\}$ forms an orthonormal basis of $\overline{\Ran S_k^*} = (\Ker S_k)^{\perp} = \caH$,
where we note that $S_k$ is injective since the support of $\mu$ is $\caX$.

Finally, Mercer's theorem~\citep{andreaschristmann2008_SupportVector,steinwart2012_MercerTheorem} yields that
\begin{align}
  \label{eq:Mercer}

  k(x,x')  = \sum_{m = 1}^{\infty} \mu_m \sum_{l = 1}^{d_m} \overline{e_{m,l}(x)} e_{m,l}(x') ,
\end{align}
where the summation converges absolutely and uniformly.
To align with the previous literature, we denote by $(\lambda_j)_{j \geq 1}$ the eigenvalues of $T$ counting multiplicities
and, with a little abuse of notation, denote by $e_j$ the corresponding eigenfunction.
Then, we introduce the following assumption on the eigenvalues of $T$, which is commonly considered in the previous literature
~\citep{caponnetto2007_OptimalRates,fischer2020_SobolevNorm,li2023_AsymptoticLearning}.

\begin{assumption}[Eigenvalue decay]
  \label{assu:EDR}
  There are some $\beta > 1$ and constants $c_{\mr{eig}}, C_{\mr{eig}} > 0$ such that
  \begin{align}
    \label{eq:EDR}
    c_{\mr{eig}} j^{- \beta} \leq \lambda_j \leq C_{\mr{eig}} j^{-\beta}, \quad j\geq 1,
  \end{align}
  or equivalently,

  \begin{align}
    \label{eq:EDR_}
    \#\{ i : \lambda_i \geq \lambda \} = \sum_{m : \mu_m \geq \lambda} d_m = \Theta(\lambda^{-1/\beta}), \qq{as} \lambda \to 0.
  \end{align}
  where $(\lambda_j)_{j\geq 1}$ are the eigenvalues (counting multiplicities) of the integral operator $T$
  and $\{\mu_m\}_{m\geq 1}$ are the distinct ones defined in \cref{eq:T_decomp}.
\end{assumption}

\begin{remark}

  The equivalence of \cref{eq:EDR} and \cref{eq:EDR_} is elementary and a proof can be found in \cref{prop:DescendSequenceEquiv}.
  \cref{eq:EDR_} allows us to deal with some less explicit cases.
  For example, if we have $\sum_{k=1}^m d_k \asymp m^{\gamma}$ for some $\gamma \geq 1$,
  then $\lambda_j \asymp j^{-\beta}$ is just equivalent to $\mu_m \asymp m^{-\gamma \beta}$.

\end{remark}

The eigenvalue decay rate in \cref{assu:EDR} is a common assumption in the previous literature~\citep{caponnetto2007_OptimalRates,fischer2020_SobolevNorm,li2023_AsymptoticLearning}.
As we aim to derive the exact generalization error, the lower bound part in \cref{eq:EDR} is also necessary.
\cref{assu:EDR} is satisfied by many commonly used kernels, such as the Laplacian kernel, Matérn kernels, and neural tangent kernels.
This assumption on the eigenvalues characterizes the smoothness of functions in the RKHS, and a larger $\beta$ implies greater smoothness.
We also remark that since $T$ is trace-class, $(\lambda_j)_{j \geq 1}$ is summable so the requirement $\beta > 1$ is necessary.
\cref{assu:EDR} is also closely connected to the effective dimension or capacity condition of the RKHS in the previous literature~\citep{caponnetto2007_OptimalRates}.
Later, with a spectral algorithm $\varphi$, we will introduce generalized $\varphi$-effective dimension that characterizes the variance of the method,
see \cref{eq:PhiEffDim}.

\paragraph{Interpolation spaces}
We need to further introduce the interpolation spaces~\citep{steinwart2012_MercerTheorem} to state our results.
For $p \geq 0$, we define the fractional power $T^p : L^2 \to L^2$ by
\begin{align}
  T^p =  \sum_{m =1}^\infty \mu_m^p P_{V_m} = \sum_{m = 1}^{\infty} \mu_m^p \sum_{l = 1}^{d_m} \ang{e_{m,l},\cdot}_{L^2} e_{m,l}.
\end{align}
Then, we can introduce the interpolation space $[\caH]^s$ by
\begin{align}
  [\caH]
  ^s =
  \Ran T^{s/2} = \left\{ f = \sum_{m = 1}^{\infty} \mu_m^{s/2} \sum_{l = 1}^{d_m} a_{m,l} e_{m,l}  ~\Big|~ \sum_{m = 1}^{\infty}\sum_{l = 1}^{d_m} \abs{a_{m,l}}^2 < \infty \right\}
  \subseteq L^2.
\end{align}
For $f,g \in [\caH]^s$ with coefficients $(a_{m,l})$ and $(b_{m,l})$ respectively, we define the inner product in $[\caH]^s$ by
\begin{align}
  \label{eq:GammaNorm}
  \ang{f,g}_{[\caH]^s} = \sum_{m = 1}^{\infty}\sum_{l = 1}^{d_m} \overline{a_{m,l}} b_{m,l} = \ang{T^{-\frac{s}{2}}f,T^{-\frac{s}{2}} g}_{L^2}.
\end{align}
Then, it is easy to see that $[\caH]^s$ is a separable Hilbert space with an orthogonal basis $\{\mu_m^{s/2} e_{m,l} :~ m\geq 1,~1\leq l \leq d_m\}$.
In particular, we have $[\caH]^0 \subseteq L^2$ and also $[\caH]^1 = \caH$.
We also have natural inclusions $[\caH]^s \subseteq [\caH]^t$ for $s \geq t$ and the inclusion is compact if $s > t$.
Moreover, the restriction of $T$ (and also $T^p$) on $[\caH]^s$ is also a bounded operator with the same spectra,
so we will still denote it by $T$ (and also $T^p$) for simplicity.

\paragraph{Regular RKHS}

To derive the sharpest possible learning rate, we need to characterize the regularity of functions in the RKHS as fully as possible.
Since $V_m$ is a finite-dimensional space of $L^2$ and also $\caH$,
it is a reproducing kernel Hilbert space with respect to $\ang{\cdot,\cdot}_{L^2}$ and its reproducing kernel $k_m$ is determined uniquely by
\begin{align}
  \ang{k_m(x,\cdot),f}_{L^2} = \int_{\caX} \overline{k_m(x,x')} f(x') \dd \mu(x') = f(x).
\end{align}
Choosing an orthonormal basis $\{e_{m,l}\}_{l=1}^{d_m}$, we have explicitly
\begin{align}
  k_m(x,x')  = \sum_{l = 1}^{d_m} \overline{e_{m,l}(x)} e_{m,l}(x'),
\end{align}
which is invariant under the choice of basis.
It is also easy to see that
\begin{align*}
  k(x,x') = \sum_{m=1}^{\infty} \mu_m k_m(x,x').
\end{align*}
In this paper, we introduce the following condition for regular RKHSs:

\begin{assumption}[Regular RKHS]
  \label{assu:RegularRKHS}
  There is some constant $M > 0$ such that
  \begin{align}
    \label{eq:RegularRKHS}
    \sup_{x \in \caX} \sum_{m=1}^N k_m(x,x) = \sup_{x \in \caX} \sum_{m=1}^N \sum_{l=1}^{d_m} \abs{e_{m,l}(x)}^2 \leq M \sum_{m=1}^N d_m,\quad
    \forall N \geq 1,
  \end{align}
  In this case, we call such an RKHS (together with the underlying distribution $\mu$) regular.
\end{assumption}

It is immediate that if the eigenfunctions are uniformly bounded, that is, $\sup_{x \in \caX}\abs{e_{m,l}(x)}^2 \leq M$, then the RKHS is regular.
Moreover, there are broader settings in which the RKHS is regular,
so we believe that it is a rather general assumption.

\begin{example}[Shift-invariant periodic kernels]
  \label{example:Torus}
  Let $\caX = \bbT^d = [-\pi,\pi)^d$ be the $d$-dimensional torus and $\mu$ be the uniform measure on $\bbT^d$.
  Consider a shift-invariant kernel satisfying $k(x,y) = h(x-y)$,
  where $h$ is a function defined on $\bbT^d$.
  Then, it is easy to show that the Fourier basis $\{ \phi_{\bm{m}} = e^{i\ang{\bm{m},x}},~ \bm{m} \in \bbZ^d \}$
  gives an orthonormal set of eigenfunctions of $T$.
  Consequently, it is regular since the basis is uniformly bounded.
  Moreover, if the corresponding eigenvalues satisfy $\lambda_{\bm{m}} \asymp (1+\norm{\bm{m}}_2^2)^{-\alpha}$,
  then the corresponding RKHS is $\caH \cong H^{\alpha}(\bbT^d)$, the Sobolev space on $\bbT^d$,
  and also $[\caH]^s \cong H^{s \alpha}(\bbT^d)$.
\end{example}

\begin{example}[Dot-product kernel on the sphere]
  \label{example:DotSphere}
  Let $\caX = \bbS^d$ be the $d$-dimensional sphere and $\mu$ be the uniform measure on $\bbS^d$.
  Consider a dot-product kernel satisfying $k(x,y) = h(\ang{x,y})$, where $h$ is a function on $[-1,1]$.
  Then, the Funk-Hecke formula~\citep{dai2013_ApproximationTheory}
  shows that the spherical harmonics $\{ Y_{m,l}~ m\geq 1,~1\leq l \leq d_m \}$ form an orthonormal set of eigenfunctions of $T$,
  where $\{ Y_{m,l}~1\leq l \leq d_m \}$ are order-$m$ homogeneous harmonic polynomials and $d_m = \binom{m+d}{m} - \binom{m-2+d}{m-2} \asymp m^{d-1}$.
  Using the theory of spherical harmonics, we can show that this RKHS is regular, see \cref{subsubsec:DotProductSphere}.
\end{example}

\begin{example}[Dot-product kernel on the ball]
  \label{example:DotBall}
  Now, let us consider the $d$-dimensional unit ball $\caX = \bbB^d = \{x \in \R^{d+1} : \norm{x} \leq 1\}$
  and let $\mu$ be proportional to the classical weight $p(x) = (1-\norm{x}^2)^{-1/2}$.
  Consider still a dot-product kernel $k$.
  Then, an analog of the Funk-Hecke formula on the ball~\citet[Section 11]{dai2013_ApproximationTheory} shows that the space $V_m^d$ of orthogonal polynomials of degree exactly $m$
  is an eigenspace associated with the same eigenvalue of $T$
  and  $\dim V_m^d = \binom{m+d-1}{m} \asymp m^{d-1}$.
  Similar to the spherical case, we can show that this RKHS is regular, see \cref{subsubsec:DotProductBall}.
\end{example}

In previous literature~\citep{steinwart2009_OptimalRates,fischer2020_SobolevNorm,zhang2023_OptimalityMisspecified},
the following $L^\infty$-embedding property has been introduced to characterize the regularity of the RKHS\@.
We say that $\caH$ has an embedding property of order $\alpha \in (0,1]$
if $[\mathcal{H}]^\alpha$ can be continuously embedded into $L^\infty(\caX,\dd \mu)$, that is,
the operator norm
\begin{equation}
  \label{eq:EMB}
  M_\alpha \coloneqq \norm{[\mathcal{H}]^\alpha \hookrightarrow L^{\infty}(\mathcal{X},\mu)} < \infty.
\end{equation}
Then, we define the embedding index by $ \alpha_0 = \inf \left\{ \alpha : M_\alpha < \infty \right\}$.
It is obvious that $\alpha \leq 1$ because
\begin{align*}
  \sup_{x \in \caX} \abs{f(x)} = \sup_{x \in \caX}\ang{k(x,\cdot),f}_{\caH}
  \leq \norm{k(x,\cdot)}_{\caH} \norm{f}_{\caH} \leq \kappa \norm{f}_{\caH},
\end{align*}
and it is also shown in \citet[Lemma 10]{fischer2020_SobolevNorm} that $\alpha_0 \geq 1/\beta$.
To obtain sharp concentrations, previous works~\citep{zhang2023_OptimalityMisspecified,li2023_KernelInterpolation} assume that $\alpha_0 = 1/\beta$.
Here we show that this embedding index condition is satisfied by the regular RKHS\@.

\begin{proposition}
  \label{prop:EmbeddingIdx}
  Under \cref{assu:EDR} and \cref{assu:RegularRKHS}, the embedding index is $\alpha_0 = 1/\beta$.
\end{proposition}

Moreover, we remark that the embedding index condition only makes sense for the eigenvalues with power-law decay,
while the regular RKHS condition can also be considered for more general decays.
In fact,  the regular RKHS condition essentially considers the eigenfunctions rather than the eigenvalues.

\subsection{Spectral algorithm}

Let $Z = \{(x_i,y_i)\}_{i=1}^n \subseteq \caX \times \caY$ be a set of training samples drawn i.i.d.\ from $\rho$.
We also denote by $X = (x_1,\dots,x_n)$ the collection of sample inputs.
To introduce the spectral algorithm, we first introduce some auxiliary notation.
Denote by $k_x = k(x,\cdot) \in \caH$.
Let $K_x : \R \to \caH$ be given by $K_x y = y k_x$, whose adjoint $K_x^* : \caH \to \R$ is given by $K_x^* f = \ang{k_x,f}_{\caH} = f(x)$.
Moreover, we denote by $T_x = K_x K_x^* : \caH \to \caH$ and
\begin{align}
  T_X = \frac{1}{n}\sum_{i=1}^n T_{x_i}.
\end{align}
Here we note that since $k(x,x) \leq \kappa^2$, we have $\norm{T_X} \leq \kappa^2$.
We also define the sample basis function
\begin{align}
  \label{eq:GZ}
  \hat{g}_Z = \frac{1}{n} \sum_{i=1}^n K_{x_i} y_i = \frac{1}{n} \sum_{i=1}^n y_i k(x_i,\cdot) \in \caH.
\end{align}
Then, a spectral algorithm is obtained by applying to $\hat{g}_Z$ a ``regularized inverse'' of $T_X$ via a filter function~\citep{bauer2007_RegularizationAlgorithms}.
We first introduce the following definition of filter functions.

\begin{definition}[Filter functions]
  \label{def:filter}

  Let $\left\{ \reg : [0,\kappa^2] \to \R_{\geq 0} \mid \lambda \in (0,1) \right\}$ be a family of functions
  indexed by the regularization parameter $\lambda$ and define the remainder function
  \begin{align}
    \label{eq:def_Psi}
    \rem(z) \coloneqq 1 - z \reg(z).
  \end{align}
  We say that $\left\{ \reg \mid \lambda \in (0,1)  \right\}$ (or simply $\reg(z)$) is a filter function if:
  \begin{enumerate}[(i)]
    \item For any fixed $\lambda$, $\rem(z) \geq 0$ is decreasing with respect to $z \in [0,\kappa^2]$.
    For any fixed $z$, $\rem(z)$ decreases as $\lambda$ decreases.
    \item There is some constant $E$ such that
    \begin{align}
      \label{eq:Filter_Reg}
      \sup_{z \in [0,\kappa^2]} (z+\lambda) \reg(z) \leq E, \quad \forall \lambda \in (0,1).

    \end{align}
    \item
    The \textit{qualification} of this filter function is $\tau_{\max} \in [1,\infty]$ such that $\forall 0 \leq \tau \leq \tau_{\max}$ (and also $\tau <\infty$),
    \begin{align}
      \label{eq:Filter_Rem}
      \sup_{z \in [0,\kappa^2]} z^\tau \rem(z)  \leq F_{\tau} \lambda^{\tau},\quad \forall \lambda \in (0,1),
    \end{align}
    where $F_\tau$ is a constant depending only on $\tau$.

    \item

    In addition, if $\tau_{\max}$ is finite, then there is some constant $\underline{F}$ that
    \begin{align*}
      \rem(z) \geq \underline{F} \lambda^{\tau_{\max}},\quad \forall z \in [0,\kappa^2],~\lambda \in (0,1).
    \end{align*}
  \end{enumerate}
\end{definition}

Now, given a filter function $\reg(z)$, a spectral algorithm is defined by
\begin{align}
  \label{eq:SA}
  \hat{f}_{\lambda} = \reg(T_X) \hat{g}_Z \quad \in \caH.
\end{align}

We postpone concrete examples of spectral algorithms to the end of this subsection.

\begin{remark}
  We remark that properties (i) and (iv) are not essential in the definition of filter functions in the literature~\citep{bauer2007_RegularizationAlgorithms,gerfo2008_SpectralAlgorithms,rosasco2005_SpectralMethods},
  but we introduce them to avoid some unnecessary technicalities in the proof.
  In particular, property (i) describes the general behavior of a spectral filter

  and property (iv) determines exactly the qualification of the filter function.

  We also remark that the regularization parameter $\lambda$ will depend on the sample size $n$, with $\lambda(n)$ going to zero as $n$ tends to infinity,
  so the upper bound of $\lambda \in (0,1)$ is not essential and can be smaller if necessary.
\end{remark}

To evaluate the performance of the spectral algorithm $\hat{f}_\lambda$, we consider the generalization error (or excess risk)~\citep{andreaschristmann2008_SupportVector} defined by
\begin{align}
  \norm{\hat{f}_\lambda - \fstar}_{L^2}^2 = \E_{x\sim \mu} \left(\hat{f}_\lambda(x) - \fstar(x)\right)^2.
\end{align}
Moreover, it is more convenient to consider its conditional expectation with respect to $X$, namely $\E \left[ \norm{\hat{f}_\lambda - \fstar}_{L^2}^2 \mid X \right] $, which is still a random variable depending on $X$.
This conditional quantity allows us to derive the bias-variance decomposition in the most natural and precise way.
If the noise $\epsilon = y - \fstar(x)$ is independent of $x$, this conditional expectation is just the expectation over the noise.
\rev{
  Passing the conditional expectation to the unconditional expectation $\E \norm{\hat{f}_\lambda - \fstar}_{L^2}^2$ only requires taking the expectation over $X$ afterward,
  while high-probability bounds on \( \norm{\hat{f}_\lambda - \fstar}_{L^2}^2 \) can also be established by applying standard concentration inequalities to the noise terms.
  Therefore, we will focus on the conditional quantity throughout the paper to streamline the presentation.
}

However, to derive the precise generalization error curve, the above definition of filter functions is not sufficient for our purposes.
The key novelty of our techniques is that we develop a special argument based on analytic functional calculus.
To this end, we introduce the following assumption on the analytic filter function.
As far as we know, we are the first to consider such properties of filter functions.

\begin{assumption}[Analytic filter function]
  \label{assu:Filter}
  Let
  \begin{align*}
    D_{\lambda} &= \left\{ z \in \bbC : \Re z \in [-\lambda/2,\kappa^2], ~ \abs{\Im z} \leq \Re z + \lambda/2 \right\} \\
    & \quad \cup \left\{ z \in \bbC : \abs{z - \kappa^2} \leq \kappa^2 + \lambda/2,~ \Re z \geq \kappa^2  \right\}.
  \end{align*}

  The filter function $\reg(z)$ can be extended to an analytic function on some domain containing $D_\lambda$
  and the following conditions hold for all $\lambda \in (0,1)$:
  \begin{enumerate}
    \item[(C1)] $\abs{(z+\lambda)\reg(z)} \leq \tilde{E},~ \forall z \in D_\lambda$;
    \item[(C2)] $\abs{(z+\lambda)\rem(z)} \leq \tilde{F} \lambda,~ \forall z \in D_\lambda$;

  \end{enumerate}
  where  $\tilde{E}, \tilde{F}$ are positive constants.

\end{assumption}

\begin{remark}
  \label{remark:Filter}
  The domain $D_\lambda$ is essential and is related to the analytic functional argument in the proof,
  see also \cref{fig:Contour} for an illustration.
  The two conditions (C1) and (C2) can be seen as the complex extension of
  \cref{eq:Filter_Reg} and \cref{eq:Filter_Rem} respectively,
  so one can expect that they also hold for the filter functions.
  Indeed, we will show below that many commonly used filter functions satisfy \cref{assu:Filter}.
  The proof is postponed to \cref{subsec:FilterFunction}.
\end{remark}

\begin{figure}[htb]
  \centering
  \subfigure{
    \begin{minipage}[t]{0.5\linewidth}
      \centering
      \includegraphics[width=1\linewidth]{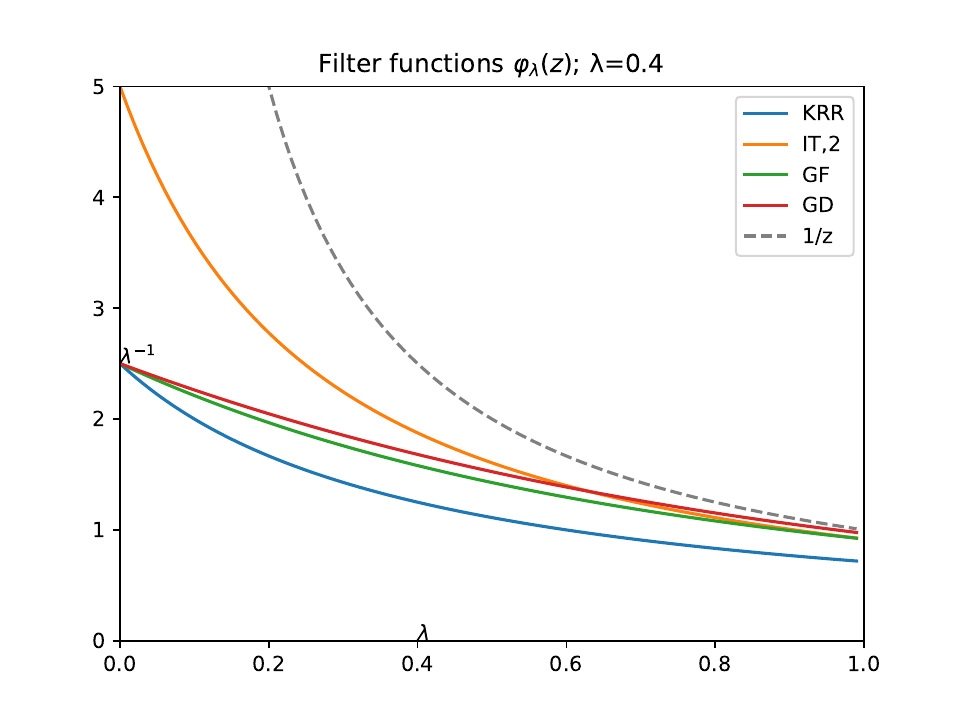}
    \end{minipage}%
  }%
  \subfigure{
    \begin{minipage}[t]{0.5\linewidth}
      \centering
      \includegraphics[width=1\linewidth]{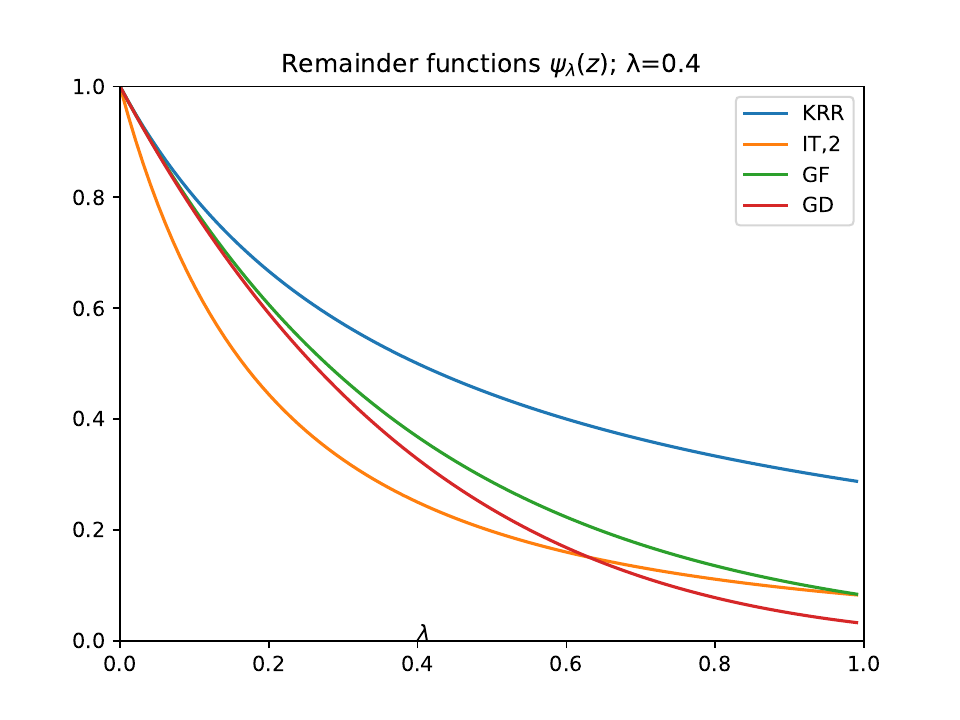}
    \end{minipage}%
  }%

  \centering
  \caption{An illustration of the filter functions $\reg$ and $\rem$.
  }

  \label{fig:Filters}
\end{figure}

\begin{example}[Kernel ridge regression]
  \label{example:KRR}
  The filter function of kernel ridge regression (KRR) is well known to be
  \begin{align}
    \reg^{\mathrm{KR}}(z) = \frac{1}{z+\lambda},\quad \rem^{\mathrm{KR}}(z) = \frac{\lambda}{z+\lambda}.
  \end{align}
  Both $\reg$ and $\rem$ are analytic on $\left\{ z \in \bbC : \Re z > -\lambda \right\} \supset D_\lambda$.
  This filter function is of qualification only $\tau_{\max} = 1$.
\end{example}

\begin{example}[Iterated ridge regression]
  \label{example:IteratedRidge}
  To overcome the limited qualification of KRR, \citet{rosasco2005_SpectralMethods} introduced the following iterated ridge (or iterated Tikhonov) method.
  Let $p \geq 1$ be fixed.
  We define
  \begin{align}
    \reg^{\mathrm{IT,p}}(z) = \frac{1}{z}\left[ 1 - \frac{\lambda^p}{(z+\lambda)^p} \right],
    \quad
    \rem^{\mathrm{IT,p}}(z) = \frac{\lambda^p}{(z+\lambda)^p}.
  \end{align}
  It is easy to show that $z = 0$ is a removable singular point of $\reg^{\mathrm{IT,p}}(z)$ and both
  $\reg^{\mathrm{IT,p}}(z)$ and $\rem^{\mathrm{IT,p}}(z)$ are analytic on $\left\{ z \in \bbC : \Re z > -\lambda \right\} \supset D_\lambda$.
  The merit of this filter function is that it has qualification $\tau_{\max} = p$.

  To understand the name ``iterated'', let us consider the particular case $p \in \bbN^*$.
  Then,
  \begin{align*}
    \reg^{\mathrm{IT,p}}(z) = \sum_{r=1}^p \lambda^{r-1} (z + \lambda)^{-r},
  \end{align*}
  and the method is obtained by iterating the ridge method $p$ times:
  \begin{align*}
    (T_X + \lambda)
    h_i = \hat{g}_Z + \lambda h_{i-1},\quad i = 1,\dots,p.

  \end{align*}
\end{example}

\begin{example}[Gradient flow]
  \label{example:GradientFlow}
  The gradient flow method~\citep{yao2007_EarlyStopping} is another popular regularization method.
  Let us consider the empirical loss
  \begin{align*}
    L(f) = \frac{1}{2n}\sum_{i=1}^n (f(x_i) - y_i)^2 = \frac{1}{2n}\sum_{i=1}^n (K_{x_i}^* f - y_i)^2,\quad f \in \caH.
  \end{align*}
  Then, with the initial value $f_0 = 0$, $\dot{f}_t = -\nabla_{f} L(f_t)$ defines a gradient flow in $\caH$.
  The gradient flow equation, which can be solved in closed form, gives the filter function
  \begin{align}
    \reg^{\mr{GF}}(z) = \frac{1-e^{-t z}}{z},\quad \rem^{\mathrm{GF}}(z) = e^{-t z}, \quad t = \lambda^{-1}.
  \end{align}
  It is also easy to show that $z = 0$ is a removable singular point of $\reg^{\mr{GF}}(z)$,
  so both $\reg^{\mr{GF}}(z)$ and $\rem^{\mr{GF}}(z)$ are analytic on the whole complex plane.
  Moreover, elementary inequalities show that the gradient flow method has qualification $\tau_{\max} = \infty$ with diverging $F_\tau = (\tau/e)^\tau$.

\end{example}

\begin{example}[Gradient descent]
  \label{example:GradientDescent}
  The gradient descent method is the discrete version of gradient flow.
  Let $\eta > 0$ be a fixed step size.
  Then, iterating gradient descent with respect to the empirical loss for $t$ steps yields the filter function
  \begin{align}
    \begin{aligned}
      \reg^{\mr{GD}}(z) &= \eta \sum_{k=0}^{t-1} (1-\eta z)^{k} = \frac{1-(1-\eta z)^t}{z},\quad \lambda = (\eta t)^{-1}, \\
      \rem^{\mr{GD}}(z) &= (1-\eta z)^t.
    \end{aligned}
  \end{align}
  Moreover, when $\eta$ is small enough, say $\eta < 1/(2\kappa^2)$, we have $\Re (1 - \eta z) > 0$ for $z \in D_\lambda$,
  so we can take the single-valued branch of $(1-\eta z)^t$ even when $t$ is not an integer.
  Therefore, we can extend the definition of the filter function so that $\lambda$ can be arbitrary and $t = (\eta \lambda)^{-1}$.
  It is also easy to show that $z = 0$ is a removable singular point of $\reg^{\mr{GD}}(z)$.
  Consequently, $\reg^{\mr{GD}}(z)$ and $\rem^{\mr{GD}}(z)$ are analytic on $D_\lambda$.
  Similar to the gradient flow method, the gradient descent method is also of qualification $\tau_{\max} = \infty$ with $F_\tau = (\tau/e)^\tau$.
\end{example}

\subsection{Notations}
We denote by $\# A$ the cardinality of a set $A$.
We use the big-O notations $O(\cdot),\Omega(\cdot), \Theta(\cdot)$, $o(\cdot)$,
as well as their probability versions $O_{\bbP}(\cdot)$, $\Omega_{\bbP}(\cdot)$, $\Theta_{\bbP}(\cdot)$ and $o_{\bbP}(\cdot)$.
Let $(\xi_n)_{n \geq 1}$ be a sequence of non-negative random variables and $(a_n)_{n \geq 1}$ a sequence of positive numbers.
We say $\xi_n = O_{\bbP}(a_n) $ if for any $\delta > 0$, there exists $N_\delta$ and $M_\delta$ such that
when $n \geq N_\delta$, $\bbP \left\{ \abs{\xi_n} \leq M_\delta a_n \right\} \geq 1-\delta$.
The notation $\Omega_{\bbP}(a_n)$ is defined similarly
and $\xi_n = \Theta_{\bbP}(a_n)$ iff $\xi_n = O_{\bbP}(a_n)$ and $\xi_n = \Omega_{\bbP}(a_n)$ both hold.
Moreover, we say $\xi_n = o_{\bbP}(a_n) $ if $\xi_n/a_n$ converges in probability to 0.
Also, we sometimes write $a_n \asymp b_n$ if $a_n = \Theta(b_n)$.

\section{Main results}
\label{sec:main-results}

\subsection{More assumptions}
\label{subsec:assumptions}
Before stating our main theorem, we introduce two assumptions.
The first assumption concerns the noise.
This assumption is quite standard and is satisfied if the noise is independent of the input $x$ and has a bounded variance.

\begin{assumption}[Noise]
  \label{assu:noise}
  We assume
  \begin{align}
    \E_{(x,y)\sim \rho} \left[ \left( y-f^{*}(x) \right)^2 \;\Big|\; x \right] = \sigma^2 > 0,
    \quad \mu\text{-a.e. } x \in \mathcal{X},
  \end{align}
\end{assumption}

The second assumption is about the regression function $\fstar$.
Recall the definition of interpolation spaces and that $\{e_{m,l} :~ m\geq 1,~1\leq l \leq d_m\}$ forms an orthogonal set in $L^2$.
Then, we first assume that the regression function admits the following expansion in the sense of $L^2$-norm:
\begin{align}
  \fstar = \sum_{m=1}^\infty \sum_{l = 1}^{d_m} f_{m,l} e_{m,l}.
\end{align}
Here we note that $\bar{f}_{m}^2 \coloneqq \sum_{l = 1}^{d_m}\abs{f_{m,l}}^2 = \norm{P_{V_m} \fstar}_{L^2}^2$ is invariant under the choice of $\{e_{m,l}\}_{l=1}^{d_m}$
which is an orthogonal basis of $V_m$.
Then, we assume that the regression function satisfies the following source condition.

\begin{assumption}[Source]
  \label{assu:Source}
  There exists some $s > 0$ such that $\fstar \in [\caH]^t$ for any $t < s$,
  $\fstar \neq 0$, and if $s < 2\tau_{\max}$,
  \begin{align}
    \sum_{m : \mu_m < \lambda} \sum_{l = 1}^{d_m} \abs{f_{m,l}}^2 =
    \sum_{m : \mu_m < \lambda} \bar{f}_{m}^2 = \Omega(\lambda^s).
  \end{align}

\end{assumption}

This assumption assumes that the regression function can be approximately described by a power-law decay with smoothness index $s$,
but it does not require that the coefficient of $\fstar$ decays exactly in a power-law manner,
which allows a wider range of regression functions to be considered.
We note that since we have to establish the exact generalization error curve, the lower bound is also necessary,
which is presented in a tail sum manner.
The following gives some examples of regression functions satisfying \cref{assu:Source},
whose proofs are deferred to \cref{subsec:SourceCondition}.

\begin{example}[Exact power-law $\fstar$]
  \label{example:source1}
  Let \cref{assu:EDR} hold and $\sum_{k=1}^m d_k \asymp m^\gamma$ for some $\gamma \geq 1$.
  Suppose $\fstar$ satisfies $\bar{f}_m \asymp m^{-(\gamma p+1)/2}$ for some $p > 1$.
  Then $\fstar$ satisfies \cref{assu:Source} for $s = \frac{p}{\beta}$.
  In particular, with a little abuse of notation, if we rearrange $(f_{m,l})$ as
  \begin{align*}
    (f_j)
    _{j \geq 1} = (f_{1,1},\dots,f_{1,d_1},f_{2,1},\dots,f_{m,1},\dots,f_{m,d_m},\dots).
  \end{align*}
  Then, this example includes the case $\abs{f_j } \asymp \lambda_j^\frac{s}{2}j^{-\hf} \asymp j^{-\frac{s\beta+1}{2}}$ considered in
  \citet{li2023_AsymptoticLearning,cui2021_GeneralizationError}.
\end{example}

We can also consider the following case that there are some gaps in the coefficients of $\fstar$.

\begin{example}
  \label{example:source2}
  In this example, we assume \cref{assu:EDR} holds

  and consider the $(f_j)_{j \geq 1}$ introduced in the previous example.
  Suppose that for some $q \geq 1, p > 1$,
  \begin{equation*}
    \begin{cases}
      \abs{f_{j(l)}} \asymp l^{-(p+1)/2}, &  l = 1,2,\dots, \\
      f_j = 0,                                     & \text{otherwise,}
    \end{cases}
  \end{equation*}
  where $j(l) \asymp l^q $.
  Then, $\fstar$ satisfies \cref{assu:Source} for $s = p/(q\beta)$.

\end{example}

\subsection{Main theorem}

Let us first introduce the two deterministic quantities that characterize the bias and the variance respectively.
We define the main bias term by
\begin{align}
  \label{eq:BiasMainTerm}
  \caR_{\varphi}^2(\lambda;\fstar) & \coloneqq \norm{\rem(T) \fstar}_{L^2}^2
  = \sum_{m=1}^\infty  \rem(\mu_m)^2 \sum_{l = 1}^{d_m} \abs{f_{m,l}}^2 = \sum_{m=1}^\infty  \rem(\mu_m)^2 \bar{f}_{m}^2.
\end{align}
For the variance term, we extend the definition of effective dimension~\citep{caponnetto2007_OptimalRates} to introduce the $\varphi$-effective dimension of order $p\geq 1$ by
\begin{align}
  \label{eq:PhiEffDim}
  \caN_{p,\varphi}(\lambda)\coloneqq \sum_{j =1}^\infty \left[ \lambda_j \reg(\lambda_j) \right]^p
  = \sum_{m=1}^\infty d_m \left[ \mu_m \reg(\mu_m) \right]^p.
\end{align}
In particular, $\varphi^{\mr{KR}}$-effective dimension of order $p=1$ is just the ordinary effective dimension considered in previous works (see \cref{eq:EffectiveDim} in the proof).

\begin{theorem}
  \label{thm:Main}
  Under Assumptions~\ref{assu:EDR},\ref{assu:RegularRKHS},\ref{assu:Filter},\ref{assu:noise} and~\ref{assu:Source},
  for any $\lambda = \lambda(n) \to 0$, we have
  \begin{itemize}
    \item If $\lambda = \Omega(n^{-\theta})$ for some $\theta < \beta$, then
    \begin{align}
      \label{eq:MainResult}
      \E \left[ \norm{\hat{f}_\lambda - \fstar}_{L^2}^2 ~\Big|~ X \right]
      = (1 + o_{\bbP}(1))\left( \caR_{\varphi}^2(\lambda;\fstar) + \frac{\sigma^2}{n} \caN_{2,\varphi}(\lambda) \right),
    \end{align}
    where $\caR_{\varphi}^2(\lambda;\fstar)$ and $\caN_{2,\varphi}(\lambda)$ are two deterministic quantities defined in
    \cref{eq:BiasMainTerm} and \cref{eq:PhiEffDim} respectively.
    \item If $\lambda = O(n^{-\beta})$, then
    \begin{align}
      \E \left[ \norm{\hat{f}_\lambda - \fstar}_{L^2}^2 ~\Big|~ X \right]
      = \Omega_{\bbP}\left( (\ln n)^{-4} \sigma^2 \right).
    \end{align}
  \end{itemize}
\end{theorem}

\subsection{Discussion}

In this subsection, we discuss our main result from the following perspectives:

\paragraph{Minimax optimal rate}
\cref{thm:Main} naturally recovers the minimax optimal rates of spectral algorithms that have been derived in previous works
~(see \citet{lin2018_OptimalRatesa,zhang2023_OptimalityMisspecifieda} as well as the references therein).
Let us suppose further that $\fstar \in [\caH]^s$ as in the standard literature and $s \leq 2\tau$.
Then, for the bias term, \cref{lem:FApproximation} shows that
\begin{align*}
  \caR_{\varphi}^2(\lambda;\fstar) = \norm{\rem(T) \fstar}_{L^2}^2 =O(\lambda^{s}).
\end{align*}
For the variance term, \cref{prop:EffectiveDimEstimation} shows that
\begin{align*}
  \frac{\sigma^2}{n}\caN_{2,\varphi}(\lambda) \asymp \frac{\sigma^2}{n} \lambda^{-1/\beta}.
\end{align*}
Consequently, choosing $\lambda \asymp n^{-\theta}$ with $\theta = \frac{1}{s\beta + 1}$ (as in the previous literature)  yields the optimal rate
\begin{align*}
  \E \left[ \norm{\hat{f}_\lambda - \fstar}_{L^2}^2 ~\Big|~ X \right] =
  O_{\bbP}\left( n^{-\frac{s\beta}{s\beta + 1}} \right).
\end{align*}

\paragraph{Exact generalization error curve}

Our result provides a complete picture of the generalization error of spectral algorithms,
showing the effect of the choice of the regularization parameter, the source condition of the regression function, the noise level, and the choice of the filter function.
In terms of regularization, as the regularization strength $\lambda$ decreases, the bias decreases while the variance increases,
showing that the \textit{bias-variance trade-off} also exists for spectral algorithms and that the learning curve is U-shaped, as one would traditionally expect.
It also suggests that a proper choice of $\lambda$ is necessary to achieve the optimal rate.

The main strength of our result is that it provides the exact $1+o_{\bbP}(1)$ form of the generalization error for a given spectral algorithm when $\lambda$ lies in the reasonable range.
In comparison, the previous works~(for example, \citet{caponnetto2007_OptimalRates,lin2018_OptimalRatesa,zhang2023_OptimalityMisspecifieda})
on minimax optimal rates can only provide upper bounds.
While their results can also partially reflect the bias-variance trade-off, this is only in an \textit{upper-bound sense}.
More precisely, they demonstrate that:
\begin{align*}
  \mr{Error} = \mr{Bias}(\lambda) + \mr{Var}(\lambda), \qq{and}
  \mr{Bias}(\lambda) \lesssim B(\lambda) \quad \mr{Var}(\lambda) \lesssim V(\lambda),
\end{align*}
and then choose the optimal \(\lambda\) by balancing \(B(\lambda)\) and \(V(\lambda)\).
They further check that, for such a choice of $\lambda$, the upper bound $B(\lambda) + V(\lambda)$ can match the minimax lower bound.

However, as $B(\lambda)$ and $V(\lambda)$ are only upper bounds, there is no guarantee that they reflect the exact bias-variance trade-off.
Moreover, the standard approach in the recent literature on spectral algorithms is to apply the so-called approximation-estimation decomposition (see, e.g., the beginning of the Proof of Theorem 4.2 in \citet{lin2018_OptimalRatesa}) to analyze the generalization error,
which directly loses the exact bias-variance trade-off.

In contrast, our results establish precise approximations of $\mr{Bias}(\lambda)$ and $\mr{Var}(\lambda)$ for any $\lambda \propto n^{-\theta}$, $\theta \in (0,\beta)$, in the $1 + o_{\bbP}(1)$ form,
providing the exact bias-variance trade-off even without losing constant factors.
While the recent work~\citep{li2023_AsymptoticLearning} rigorously proved the learning curves of KRR, it provided only asymptotic rates (in the form of $\Theta_{\bbP}(n^{-r})$),
so our result is also a refinement of their result even in the KRR case.
For general spectral algorithms,
as far as we know, we are the first to provide such a $1+o_{\bbP}(1)$ form of exact generalization error curves.

The implication of the $1+o_{\bbP}(1)$ form of exact generalization error is that it allows us to consider the constant factor in the generalization error.
It can precisely reflect how the magnitude of the regression function and the noise affect the generalization error.
Moreover, from the oracle viewpoint, minimizing the sum of the two terms in \cref{eq:MainResult}
yields the best choice of $\lambda = \lambda(n)$ and the best generalization error, going beyond merely the asymptotic rate.

\paragraph{Interpolating regime}
We refer to the case of weak regularization, namely $\lambda = \lambda(n) = O(n^{-\beta})$, as the \textit{interpolating regime}.
\cref{thm:Main} shows that in this regime, the generalization error is of order $\Omega_{\bbP}((\ln n)^{-4})$,
which is nearly constant, so the estimator does not generalize at all.

The performance of kernel methods in the interpolating regime is also considered in the previous literature.
Under restricted settings, several works
~\citep{rakhlin2018_ConsistencyInterpolation,buchholz2022_KernelInterpolation,beaglehole2022_KernelRidgeless}
showed the inconsistency of the kernel minimum-norm interpolator, which is the $\lambda \to 0^+$ limit of spectral algorithms.
Most relevant to this paper,
\citet{li2023_KernelInterpolation} showed that for general RKHSs associated with a Hölder continuous kernel satisfying the embedding index condition,
the generalization error of KRR in the interpolating regime is $\Omega_{\bbP}(n^{-\ep})$ for any $\ep > 0$.
In comparison, we remove the condition of Hölder continuity and
also provide an improved lower bound $\Omega_{\bbP}((\ln n)^{-4})$ for spectral algorithms in the interpolating regime
by refining the analysis using the condition of regular RKHS\@.
We believe that this improved lower bound further confirms that kernel methods do not generalize in the interpolating regime,
highlighting the necessity of the regularization.

\rev{
  Notably, another line of work (see, e.g., \citep{liang2020_JustInterpolate}) shows that the minimum-norm interpolator can generalize well, but these results are proved in a different asymptotic regime where the input dimension grows with $n$ and the kernel can effectively be linearized, while here we work in a fixed dimension setting and study the whole interpolating regime $\lambda = O(n^{-\beta})$ for general spectral algorithms.
  Identifying the phase transition between the non-generalization and generalization of the minimum-norm interpolator is an interesting future direction.
}

\paragraph{Saturation effect of higher order}

The saturation effect refers to the phenomenon that, for a certain spectral algorithm with limited qualification,
it cannot achieve the minimax optimal rate of convergence when the smoothness of the regression function exceeds its qualification.
Since the traditional literature only provides the upper bound of the generalization error, they cannot prove the saturation effect.
The recent work \citet{li2023_SaturationEffect} rigorously proved the saturation effect for KRR, whose qualification is $\tau_{\max} = 1$,
but there is still no result for other spectral algorithms with higher (but limited) qualification,
such as the iterated ridge regularization (see \cref{example:IteratedRidge}).

Thanks to the exact generalization error curve, we can prove the saturation effect for spectral algorithms with higher qualification.
Let us consider a spectral algorithm with limited qualification $\tau_{\max} < \infty$, which is the case of the iterated ridge regularization.
Then, for $\fstar \in [\caH]^s$ with $s \geq 2\tau_{\max}$ and $\fstar \neq 0$, it is easy to see that
\begin{align*}
  \caR_{\varphi}^2(\lambda;\fstar) = \sum_{m=1}^\infty  \rem(\mu_m)^2 \bar{f}_{m}^2
  \geq \sum_{m=1}^\infty \underline{F} \lambda^{2\tau} \bar{f}_{m}^2 = \underline{F} \norm{\fstar}_{L^2}^2 \lambda^{2\tau}.
\end{align*}
Consequently, with the upper bound in \cref{lem:FApproximation},
we conclude that $\caR_{\varphi}^2(\lambda;\fstar) \asymp \lambda^{2\tau}$.

Then, when $\lambda$ is not too small, the main theorem gives
\begin{align*}
  \E \left[ \norm{\hat{f}_\lambda - \fstar}_{L^2}^2 ~\Big|~ X \right]
  &= (1 + o_{\bbP}(1))\left( \caR_{\varphi}^2(\lambda;\fstar) + \frac{\sigma^2}{n} \caN_{2,\varphi}(\lambda) \right) \\
  &= \Omega_{\bbP}\left(\lambda^{2\tau} + \frac{\sigma^2}{n}\lambda^{-1/\beta} \right)
  =\Omega_{\bbP}\left(n^{-\frac{2\tau \beta}{2\tau \beta+1}} \right).
\end{align*}
To cover the case when $\lambda$ can possibly be too small, we can consider $\bar{\lambda} = \max(\lambda, n^{-\theta})$ for some $\theta < \beta$
and apply the monotonicity of the variance term (\cref{lem:VarianceMonotone}) as in \citet{li2023_SaturationEffect}.
We formulate it as a corollary.
\begin{corollary}[Saturation effect]
  Suppose Assumptions~\ref{assu:EDR},\ref{assu:RegularRKHS},\ref{assu:Filter},\ref{assu:noise} hold.
  Let $\reg$ be a filter function with qualification $\tau = \tau_{\max} < \infty$.
  Then, for any non-zero $\fstar \in [\caH]^s$ for $s \geq 2\tau$, for any choice of $\lambda = \lambda(n) \to 0$,
  we have
  \begin{align}
    \E \left[ \norm{\hat{f}_\lambda - \fstar}_{L^2}^2 ~\Big|~ X \right]
    = \Omega_{\bbP}\left(n^{-\frac{2\tau \beta}{2\tau \beta+1}} \right).
  \end{align}
  Moreover, the lower bound is attained when $\lambda \asymp n^{-\theta}$ for $\theta = \frac{\beta}{2\tau \beta + 1}$.
\end{corollary}

\paragraph{The analytic functional argument}
As one of our technical contributions, we develop an \textit{analytic functional argument} based on analytic functional calculus in the context of spectral algorithms,
which enables us to derive the exact generalization error curve.
While functional calculus has a long history in operator perturbation theory~\citep{hsing2015_TheoreticalFoundations},
as far as we know, we are the first to apply it to analyze spectral algorithms.
First, we illustrate the difficulties here and explain why the existing techniques are not applicable.

The traditional literature on optimal rates focused only on upper bounds, and its approaches, which are based on the approximation-estimation decomposition
~(for example, Eq. (88) in \citet{zhang2023_OptimalityMisspecifieda}), are not applicable to lower bounds.
Moreover, it is in general more difficult to provide a lower bound than an upper bound, since the former requires the error term to be infinitesimal.
For the simple case of KRR, the rigorous work~\citep{li2023_AsymptoticLearning} determined the asymptotic rate of convergence,
but the proof method must rely on the \textit{resolvent identity} of KRR, that is,
\begin{align*}
  \reg^{\mr{KR}}(A) - \reg^{\mr{KR}}(B)
  =  (A+\lambda)^{-1} - (B+\lambda)^{-1}
  = (A+\lambda)^{-1} (B-A)(B+\lambda)^{-1},
\end{align*}
and also $\rem^{\mr{KR}}(A) = \lambda (A+\lambda)^{-1}$.
This identity is crucial for concentrating the random terms (see \cref{eq:BiasVarDecomposition}) to the non-random counterpart appearing
on the right hand side of \cref{eq:MainResult},
where we will encounter quantities like $\reg(T) - \reg(T_X)$ and $\rem(T) - \rem(T_X)$.
However, for general spectral algorithms, it is impossible to derive a similar identity.
Moreover, the effect of $\lambda$ must also be taken into consideration.

Our ``analytic functional argument'' overcomes these difficulties using analytic functional calculus:
\begin{align*}
  \reg(A) - \reg(B) &= -\frac{1}{2\pi i}  \oint_{\Gamma} \reg(z) \left[R_A(z) - R_B(z) \right] \dd z \\
  &= \frac{1}{2\pi i}\oint_{\Gamma} R_B(z) (A-B) R_A(z) \reg(z)\dd z,
\end{align*}
where $R_A(z) = (A-z)^{-1}$ is the resolvent of $A$ and $\Gamma$ is a contour.
Then, the terms in the integral resemble those of KRR (but note that now $z$ is a complex number).
Surprisingly, with a carefully chosen contour $\Gamma$ depending on $\lambda$,
this crucial formula allows us to apply the concentration results obtained for the resolvent and derive very sharp estimates,
leading to the exact generalization error curve.

We believe that this novel technique is of independent interest and can be applied to other problems.

\section{Proof}

\subsection{Proof sketch}
\label{subsec:proof_sketch}

The proof idea is quite direct.
The first step is the traditional bias-variance decomposition, which is also standard in the literature~\citep{li2023_AsymptoticLearning}.
Let us first define some quantities derived from conditioning on the sample points $X$:
\begin{align}
  \gtl &\coloneqq \mathbb{E}\left( \hat{g}_{Z} \big| X \right) = \frac{1}{n} \sum_{i=1}^n K_{x_i} \fstar (x_{i}) \in \mathcal{H}, \\
  \label{eq:tilde_f}
  \tilde{f}_{\lambda} &\coloneqq \mathbb{E}\left( \hat{f}_{\lambda} \big| X \right) = \reg(T_X) \gtl \in \mathcal{H}.
\end{align}

\begin{proposition}
  \label{prop:BiasVar}
  Under \cref{assu:noise}, we have
  \begin{align}
    \begin{aligned}
      \label{eq:BiasVarDecomposition}
      \E \left[ \norm{\hat{f}_\lambda - f^*}^2_{L^2} \;\Big|\; X \right]
      &= \norm{ \tilde{f}_{\lambda} - \fstar}^2_{L^2}  +
      \frac{\sigma^2}{n^2} \sum_{i=1}^n  \norm{\reg(T_X) k(x_i,\cdot)}^2_{L^2} \\
      & \eqqcolon  \mathbf{Bias}^2(\lambda) + \mathbf{Var}(\lambda),
    \end{aligned}
  \end{align}
  where we note that both $\mathbf{Bias}^2(\lambda)$ and $\mathbf{Var}(\lambda)$ are still random variables depending on $X$.
\end{proposition}

Then, we will show in \cref{thm:Variance} and \cref{thm:Bias}, respectively, that for $\lambda = \Omega(n^{-\theta})$, $\theta < \beta$,
\begin{align*}
  \mathbf{Var}(\lambda) &= (1+o_{\bbP}(1)) \frac{\sigma^2}{n} \caN_{2,\varphi}(\lambda), \\
  \mathbf{Bias}^2(\lambda) &= \caR_{\varphi}^2(\lambda;\fstar) +o_{\bbP}(1)\left(\caR_{\varphi}^2(\lambda;\fstar)+\frac{1}{n}\caN_{2,\varphi}(\lambda)\right).
\end{align*}
Moreover, in \cref{cor:Interpolation}, using the monotonicity of $\mathbf{Var}(\lambda)$, we can also provide a lower bound of $\mathbf{Var}(\lambda)$ when $\lambda = O(n^{-\beta})$.
Then, pulling everything together finishes the proof of \cref{thm:Main}.

\paragraph{Organization}
In the following, we first give a simple proof of the bias-variance decomposition in \cref{prop:BiasVar}.
For ease of reference, a notation table collecting the main symbols used in the paper is provided in \cref{tab:notation}.
In \cref{subsec:fundamental_controls}, we will derive estimates of some fundamental quantities that will be used later.
In \cref{subsec:concentrations}, we will use concentration inequalities to obtain high-probability bounds on some intermediate but crucial quantities.
In \cref{subsec:analytic_functional_calculus}, we recall some basic facts about analytic functional calculus
and define the contour that is essential in the proof.
Finally, we prove the estimates for the two terms in \cref{subsec:variance_term} and \cref{subsec:bias_term}.

\paragraph{Notation}
In the proof, we will denote by $C,c$ generic positive constants that may change from line to line.
We use $\norm{\cdot}_{\mathscr{B}(H)}$ or simply $\norm{\cdot}$ to represent the operator norm on a Hilbert space $H$.
We also denote by $x^+ = \max(x,0)$.

\begin{proof}[Proof of \cref{prop:BiasVar}]
  Let $\epsilon_i = y_i - \fstar(x_i)$ be the noise.
  Then, plugging \cref{eq:GZ} into \cref{eq:SA}, we get
  \begin{align*}
    \hat{f}_{\lambda} &= \reg(T_X) \hat{g}_Z  = \reg(T_X) \left( \frac{1}{n} \sum_{i=1}^n K_{x_i} (\fstar(x_i) + \epsilon_i) \right) \\
    &= \reg(T_X) \frac{1}{n} \sum_{i=1}^n K_{x_i}\fstar(x_i) + \frac{1}{n} \sum_{i=1}^n \epsilon_i \reg(T_X) k_{x_i} \\
    &= \tilde{f}_{\lambda} + \frac{1}{n} \sum_{i=1}^n \epsilon_i \reg(T_X) k_{x_i},
  \end{align*}
  so
  \begin{align*}
    \hat{f}_{\lambda} - \fstar = \tilde{f}_{\lambda} -\fstar + \frac{1}{n} \sum_{i=1}^n \epsilon_i \reg(T_X) k_{x_i}.
  \end{align*}
  Taking conditional expectation with respect to $X$ and using the fact that $\{\epsilon_i\}_{i=1}^n$ are i.i.d.\ mean-zero random variables conditional on $X$, together with \cref{assu:noise}, we get
  \begin{align*}
    \E \left[ \norm{\hat{f}_\lambda - f^*}^2_{L^2} \;\Big|\; X \right]
    &= \E \left[ \ang{\hat{f}_\lambda - f^*,~ \hat{f}_\lambda - f^*}_{L^2} \;\Big|\; X \right] \\
    &= \norm{ \tilde{f}_{\lambda} - \fstar}^2_{L^2} \\
    &\quad + 2 \E \left[
                    \Re \ang{\tilde{f}_{\lambda} - \fstar,~ \frac{1}{n} \sum_{i=1}^n \epsilon_i \reg(T_X) k_{x_i}}_{L^2}
                    \;\Big|\; X
    \right] \\
    &\quad + \E \left[ \norm{\frac{1}{n} \sum_{i=1}^n \epsilon_i \reg(T_X) k_{x_i}}^2_{L^2} \;\Big|\; X \right] \\
    &= \norm{ \tilde{f}_{\lambda} - \fstar}^2_{L^2} \\
    &\quad + \frac{1}{n^2}\E \left[
                               \ang{\sum_{i=1}^n \epsilon_i \reg(T_X) k_{x_i},~ \sum_{i=1}^n \epsilon_i \reg(T_X) k_{x_i}}_{L^2}
                               \;\Big|\; X
    \right] \\
    &= \norm{ \tilde{f}_{\lambda} - \fstar}^2_{L^2} + \frac{1}{n^2} \sum_{i=1}^n \E(\epsilon_i^2 | X) \norm{\reg(T_X) k(x_i,\cdot)}^2_{L^2} \\
    &= \norm{ \tilde{f}_{\lambda} - \fstar}^2_{L^2} + \frac{\sigma^2}{n^2} \sum_{i=1}^n \norm{\reg(T_X) k(x_i,\cdot)}^2_{L^2}.
  \end{align*}
\end{proof}

\subsection{Fundamental controls}
\label{subsec:fundamental_controls}

Denote the effective dimension (of power $p \geq 1$) of the RKHS $\caH$ by
\begin{align}
  \label{eq:EffectiveDim}
  \caN_p(\lambda) \coloneqq \sum_{i =1}^\infty \left( \frac{\lambda_i}{\lambda+\lambda_i} \right)^p.
\end{align}
This quantity corresponds to the $\varphi^{\mr{KR}}$-effective dimension defined previously in \cref{eq:PhiEffDim}.
When $p=1$, it is the ordinary effective dimension in the literature~\citep{caponnetto2007_OptimalRates}.

Let us provide first the controls of the regularized basis functions using the regular RKHS condition.

\begin{lemma}
  \label{lem:NormsN1N2}
  Under \cref{assu:RegularRKHS}, for any $x \in \caX$,
  \begin{align}
    \label{eq:RegPhiKx}
    \norm{\reg^{1/2}(T) k_x}_{\caH}^2 & \leq M \caN_{1,\varphi}(\lambda), \qquad    \norm{\reg(T) k_x}_{L^2}^2  \leq M \caN_{2,\varphi}(\lambda).
  \end{align}
  In particular, for $\reg = \reg^{\mr{KR}}$,
  \begin{align}
    \label{eq:RegKx}
    \norm{T_\lambda^{-1/2} k_x}_{\caH}^2 & \leq M \caN_1(\lambda), \qquad \norm{T_\lambda^{-1} k_x}_{L^2}^2  \leq M \caN_2(\lambda).
  \end{align}
\end{lemma}
\begin{proof}
  Using the spectral decomposition \cref{eq:T_decomp} and Mercer's decomposition \cref{eq:Mercer}, we have
  \begin{align*}
    \norm{\reg^{1/2}(T) k_x}_{\caH}^2 &=
    \norm{\reg^{1/2}(T) \sum_{m=1}^\infty  \mu_m \sum_{l=1}^{d_m} \overline{e_{m,l}(x)} e_{m,l} }_{\caH}^2 \\
    &= \norm{\sum_{m=1}^\infty \reg^{1/2}(\mu_m) \mu_m \sum_{l=1}^{d_m} \overline{e_{m,l}(x)} e_{m,l} }_{\caH}^2 \\
    &= \sum_{m=1}^\infty \mu_m \reg(\mu_m) \sum_{l=1}^{d_m}  \abs{e_{m,l}(x)}^2.
  \end{align*}
  Then, noticing that $a_m = \mu_m \reg(\mu_m)$ is decreasing since $\mu_m$ is decreasing and $t\reg(t)$ decreases as $t$ decreases,
  \cref{prop:SummationOrder} and \cref{assu:RegularRKHS} yield
  \begin{align*}
    \norm{\reg^{1/2}(T) k_x}_{\caH}^2
    \leq M \sum_{m=1}^\infty \mu_m \reg(\mu_m) d_m = M \caN_{1,\varphi}(\lambda).
  \end{align*}

  Similarly,
  \begin{align*}
    \norm{\reg(T) k_x}_{L^2}^2
    = \sum_{m=1}^\infty \left( \mu_m \reg(\mu_m) \right)^2 \sum_{l=1}^{d_m}  \abs{e_{m,l}(x)}^2
    \leq M \caN_{2,\varphi}(\lambda).
  \end{align*}
\end{proof}

Under the power-law decay \cref{assu:EDR}, we have the following asymptotics of the effective dimension,
see \cref{prop:EffectiveDimEstimation} in the appendix for the proof.

\begin{lemma}
  \label{lem:EffectiveDimEstimationPowerlaw}
  Let \cref{assu:EDR} hold and $\reg$ be a filter function satisfying \cref{eq:Filter_Reg} and \cref{eq:Filter_Rem}.
  Then, for any $p \geq 1$ and $\lambda > 0$, we have
  \begin{align}
    \mathcal{N}_{p,\varphi}(\lambda) \asymp \lambda^{-1/\beta}.
  \end{align}
  Particularly, for $\reg = \reg^{\mr{KR}}$, we have $\mathcal{N}_{p}(\lambda) \asymp \lambda^{-1/\beta}$.
\end{lemma}

The following lemma controls the residual term, which will be used in the proof of the bias term.

\begin{lemma}
  \label{lem:FApproximation}
  Suppose $f^* \in [\caH]^t$.
  Let $\reg$ be a filter function of qualification $\tau$ and $f^*_\lambda = \reg(T)Tf^*$.
  Then, for $\gamma \in [0,t]$,
  \begin{align}
    \norm{f^* - f^*_\lambda}_{[\caH]^\gamma} = \norm{\rem(T) f^*}_{[\caH]^\gamma} \leq
    F_{\theta} \kappa^{(t-\gamma-2\tau)^+} \norm{f^*}_{[\caH]^t} \lambda^{\theta},
  \end{align}
  where $\theta = \min\left(\frac{t-\gamma}{2},\tau\right)$.
  In addition, for $\gamma \in [0,2+t]$,
  \begin{align}
    \norm{f^*_\lambda}_{[\caH]^\gamma} \leq E \kappa^{(t-\gamma)^+} \norm{f^*}_{[\caH]^t} \lambda^{-\frac{(\gamma-t)^+}{2}}.
  \end{align}
\end{lemma}
\begin{proof}
  Since $f^* \in [\caH]^t$, we can find $h \in L^2$ such that $f^* = T^{t/2} h$ and $\norm{h}_{L^2} = \norm{f^*}_{[\caH]^t}$.
  Then,
  \begin{align*}
    \norm{\rem(T) f^*}_{[\caH]^\gamma} & =
    \norm{T^{-\gamma/2} \rem(T) T^{t/2} h}_{L^2}
    \leq \norm{T^{\frac{t-\gamma}{2}} \rem(T)}_{\mathscr{B}(L^2)} \norm{h}_{L^2} \\
    &\leq \norm{T^{\frac{(t-\gamma-2\tau)^+}{2}}} \norm{T^{\theta} \rem(T)}_{\mathscr{B}(L^2)} \norm{f^*}_{[\caH]^t} \\
    & \leq F_{\theta} \kappa^{(t-\gamma-2\tau)^+} \norm{f^*}_{[\caH]^t} \lambda^{\theta},
  \end{align*}
  where $\theta =\min\left(\frac{t-\gamma}{2},\tau\right)$.
  For the second inequality, we have
  \begin{align*}
    \norm{f^*_\lambda}_{[\caH]^\gamma} & = \norm{T^{-\frac{\gamma}{2}} T \reg(T) T^{\frac{t}{2}} h}_{L^2} \\
    & \leq \norm{T^{\frac{2+t-\gamma}{2}} \reg(T)}_{\mathscr{B}(L^2)} \norm{h}_{L^2} \\
    & \leq \norm{T^{\frac{(t-\gamma)^+}{2}}} \norm{T^{1-\frac{(\gamma-t)^+}{2}} \reg(T)}_{\mathscr{B}(L^2)} \norm{h}_{L^2} \\
    & \leq E \kappa^{(t-\gamma)^+} \norm{f^*}_{[\caH]^t} \lambda^{-\frac{(\gamma-t)^+}{2}},
  \end{align*}
  where we use \cref{eq:Filter_Regs} for the second term in the last inequality.
\end{proof}

\begin{proposition}
  \label{prop:ResidualTermLowerBound}
  Let $\fstar \in [\caH]^0$.
  Suppose $\reg$ is a filter function defined in \cref{def:filter} with qualification $\tau_{\max}$.
  Then,
  \begin{align*}
    R_{\varphi}^2(\lambda;\fstar) \geq \frac{1}{4} \sum_{ m : \mu_m < \frac{\lambda}{2E}} \bar{f}_m^2
  \end{align*}
  Moreover, if $\tau_{\max} < \infty$, we also have $R_{\varphi}^2(\lambda;\fstar) \geq \norm{f^*}_{L^2}^2 \underline{F} \lambda^{2\tau_{\max}}$.

  Consequently, under \cref{assu:Source}, we have
  \begin{align}
    R_{\varphi}^2(\lambda;\fstar) = \Omega(\lambda^{\min(s,2\tau_{\max})}).
  \end{align}
\end{proposition}
\begin{proof}
  Recall the definition of $R_{\varphi}^2(\lambda;\fstar)$ in \cref{eq:BiasMainTerm}.
  For the first estimation, using the last control in \cref{lem:Filter_MoreControl}, when $\mu_m \leq \frac{\lambda}{2E}$, we have $\rem(\mu_m) \geq 1/2$,
  so
  \begin{align*}
    R_{\varphi}^2(\lambda;\fstar)
    = \sum_{m=1}^\infty \rem(\mu_m)^2 \bar{f}_{m}^2
    \geq \sum_{m : \mu_m < \frac{\lambda}{2E}} \rem(\mu_m)^2 \bar{f}_{m}^2
    \geq \frac{1}{4} \sum_{m : \mu_m < \frac{\lambda}{2E}}\bar{f}_{m}^2.
  \end{align*}

  For the second one, using the property (iv) in \cref{def:filter}, we have
  \begin{align*}
    R_{\varphi}^2(\lambda;\fstar) = \sum_{m=1}^\infty  \rem(\mu_m)^2 \bar{f}_{m}^2
    \geq \sum_{m=1}^\infty \underline{F} \lambda^{2\tau_{\max}}\bar{f}_{m}^2
    = \underline{F} \lambda^{2\tau_{\max}} \norm{f^*}_{L^2}^2.
  \end{align*}
\end{proof}

\subsection{Concentrations}
\label{subsec:concentrations}

Under the regular RKHS condition,
the following inequality refines the corresponding concentration inequality in the previous literature~\citep[Proposition 5.8]{li2023_KernelInterpolation}.
The main improvement is that the quantity $\caN_1(\lambda) \asymp \lambda^{-1/\beta}$ appearing in the right-hand side is strictly smaller than
the quantity $M_\alpha \lambda^{-\alpha}$, $\alpha > 1/\beta$ appearing in their bound, which diverges as $\alpha \to 1/\beta$.

\begin{proposition}
  \label{prop:ConcenRaw}
  Under \cref{assu:RegularRKHS},
  for all $\delta \in (0,1)$, with probability at least $1 - \delta$,
  \begin{equation}
    \norm{T_{\lambda}^{-\hf} (T - T_X) T_{\lambda}^{-\hf} }_{\mathscr{B}(\caH)}
    \leq \frac{2}{3} u + \sqrt {u},
  \end{equation}
  where
  \begin{equation}
    \label{eq:Quantity_u}
    u = u(n,\lambda) = \frac{2M \caN_1(\lambda)}{n} \ln{\frac{4 (\norm{T} + \lambda) \mathcal{N}_1(\lambda)  }{\delta \norm{T}}}.
  \end{equation}

\end{proposition}

\begin{proof}
  We prove this by applying \cref{lem:ConcenBernstein}.
  Let us define
  \begin{align*}
    A(x) = T_\lambda^{-\hf}(T_x - T)T_\lambda^{-\hf}
  \end{align*}
  and $A_i = A(x_i)$.
  Then, it is easy to see that $\E (A_i) = 0$ and
  \begin{align*}
    \frac{1}{n}\sum_{i=1}^n A_i = T_\lambda^{-\hf}(T_X - T)T_\lambda^{-\hf},
  \end{align*}
  which is the quantity of interest.
  Moreover, since
  \begin{align*}
    T_\lambda^{-\hf} T_x T_\lambda^{-\hf}
    = T_\lambda^{-\hf} K_x K_x^* T_\lambda^{-\hf} =
    T_\lambda^{-\hf} K_x \left[ T_\lambda^{-\hf} K_x \right]^*,
  \end{align*}
  from \cref{eq:RegKx} we have
  \begin{align}
    \label{eq:Proof_OpNormBound}
    \norm{T_\lambda^{-\hf} T_x T_\lambda^{-\hf}}_{\mathscr{B}(\caH)} =
    \norm{T_\lambda^{-\hf} K_x}_{\mathscr{B}(\R,\caH)}^2
    = \norm{T_\lambda^{-\hf} k_x}_{\caH}^2 \leq M \caN_1(\lambda).
  \end{align}
  By taking expectation,
  we also have $\norm{T_\lambda^{-\hf} T T_\lambda^{-\hf}}_{\mathscr{B}(\caH)} \leq M \caN_1(\lambda)$.
  Therefore, we get
  \begin{align*}
    \norm{A}_{\mathscr{B}(\caH)} \leq
    \norm{T_\lambda^{-\hf} T T_\lambda^{-\hf}}_{\mathscr{B}(\caH)} +
    \norm{T_\lambda^{-\hf} T_x T_\lambda^{-\hf}}_{\mathscr{B}(\caH)}
    \leq 2M \caN_1(\lambda) \eqqcolon L.
  \end{align*}

  For the second part of the condition,
  using the fact that $\E (B - \E (B) )^2 \preceq \E (B^2)$ and also $B^2 \preceq \norm{B} B$ for a positive self-adjoint operator $B$,
  where $\preceq$ denotes the partial order induced by positive operators,
  we have
  \begin{align*}
    \E (A^2) \preceq \E \left( T_\lambda^{-\hf} T_x T_\lambda^{-\hf} \right)^2
    \preceq  L \E \left( T_\lambda^{-\hf}T_x T_\lambda^{-\hf} \right)
    = L T T_\lambda^{-1} \eqqcolon V,
  \end{align*}
  where the second $\preceq$ comes from \cref{eq:Proof_OpNormBound}.
  Therefore,
  \begin{align*}
    \norm{V} &= L \norm{TT_\lambda^{-1}} =
    L \frac{\lambda_1}{\lambda+\lambda_1},\qquad
    \Tr V  =  L  \Tr \left[ TT_\lambda^{-1} \right]
    = L \caN_1(\lambda), \\
    B &= \ln \frac{4 \Tr V}{\delta \norm{V}} = \ln \frac{4(\lambda_1+\lambda) \mathcal{N}_1(\lambda)}{\delta \lambda_1}.
  \end{align*}
  Finally, we note that the quantities in the lemma are:
  \begin{align*}
    \frac{2LB}{3n} = \frac{4M \caN_1(\lambda) B}{3n} = \frac{2}{3}u,\qquad
    \frac{2\norm{V}B}{n} \leq \frac{2M \caN_1(\lambda) B}{n} = u.
  \end{align*}
\end{proof}

The next lemma follows from \cref{prop:ConcenRaw}.

\begin{lemma}
  \label{lem:Concen}
  Suppose \cref{assu:RegularRKHS} holds.
  Fix $\delta \in (0,1)$.
  Let us denote
  \begin{align}
    \label{eq:QuantityV}
    v = v(n,\lambda) = \frac{M \caN_1(\lambda)}{n} \ln \frac{\caN_1(\lambda)}{\delta}.
  \end{align}
  Suppose $\lambda = \lambda(n) \to 0$ satisfies $v(n,\lambda) = o(1)$.
  Then, when $n$ is sufficiently large, with probability at least $1-\delta$ we have
  \begin{gather}
    \label{eq:ConcenU}
    \norm{T_{\lambda}^{-1/2} (T - T_X) T_{\lambda}^{-1/2} } \leq C \sqrt {v}, \\
    \label{eq:ConcenRatio}
    \begin{aligned}
      & \norm{T_{X\lambda}^{-1/2} T_{\lambda}^{1/2}}^2 = \norm{T_{\lambda}^{1/2} T_{X\lambda}^{-1/2} }^2 =
      \norm{T_{\lambda}^{1/2}T_{X\lambda}^{-1} T_{\lambda}^{1/2}}\leq 3, \\
      & \norm{T_{\lambda}^{-1/2} T_{X\lambda}^{1/2} }^2 = \norm{T_{X\lambda}^{1/2} T_{\lambda}^{-1/2}}^2
      = \norm{T_{\lambda}^{-1/2} T_{X\lambda}^{1}T_{\lambda}^{-1/2} } \leq 2,
    \end{aligned}
  \end{gather}
  where $C$ is an absolute constant.
\end{lemma}
Combining with \cref{lem:EffectiveDimEstimationPowerlaw}, we have the following corollary.
\begin{corollary}
  \label{cor:Concen}
  Suppose \cref{assu:EDR} and \cref{assu:RegularRKHS} hold.
  Then, as long as $\lambda(n) = \Omega(n^{-\theta})$ for some $\theta < \beta$,
  for fixed $\delta \in (0,1)$, we have $v(n,\lambda) = o(1)$, so the conclusion in \cref{lem:Concen} holds.
\end{corollary}
\begin{proof}[Proof of \cref{lem:Concen}]
  \cref{eq:ConcenU} is a direct corollary of \cref{prop:ConcenRaw} with $v = o(1)$
  and
  \begin{align*}
    u = \frac{2M \caN_1(\lambda)}{n} \ln{\frac{4 (\lambda_1 + \lambda) \mathcal{N}_1(\lambda)  }{\delta \lambda_1}}
    \leq C_0 v
  \end{align*}
  for some absolute constant $C_0$.
  For the second part, when $n$ is sufficiently large that $u \leq 1/4$,
  \begin{align*}
    \norm{T_{\lambda}^{-1/2} (T - T_X) T_{\lambda}^{-1/2} }
    \leq \frac{2}{3} u + \sqrt {u} \leq \frac{2}{3}.
  \end{align*}
  Noticing that $\left( T_{X\lambda}^{-1/2} T_{\lambda}^{1/2} \right)^* = T_{\lambda}^{1/2}T_{X\lambda}^{-1/2}$
  and $\norm{A}^2 = \norm{A^*}^2 = \norm{A^* A}$,
  we have
  \begin{align*}
    \norm{T_{X\lambda}^{-1/2} T_{\lambda}^{1/2}}^2
    = \norm{T_{\lambda}^{1/2} T_{X\lambda}^{-1/2}}^2
    &= \norm{T_\lambda^{1/2} (T_X +\lambda)^{-1}T_\lambda^{1/2}} \\
    &= \norm{\left[ T_\lambda^{-1/2} (T_X +\lambda)T_\lambda^{-1/2} \right]^{-1}} \\
    &= \norm{\left[ I - T_\lambda^{-1/2} (T - T_X)T_\lambda^{-1/2} \right]^{-1}} \\
    & \leq \left[1-\norm{ T_{\lambda}^{-1/2} (T - T_X) T_{\lambda}^{-1/2} }\right]^{-1} \leq 3,
  \end{align*}
  where in the last inequality we use the fact that $\norm{(I-A)^{-1}} \leq (1-\norm{A})^{-1}$.

  For the other part, we have
  \begin{align*}
    \norm{T_{\lambda}^{-1/2} T_{X\lambda}^{1/2} }^2 =
    \norm{T_{X\lambda}^{1/2} T_{\lambda}^{-1/2}  }^2
    & = \norm{T_{\lambda}^{-1/2} T_{X\lambda} T_{\lambda}^{-1/2}   } \\
    &= \norm{I + T_{\lambda}^{-1/2} (T_X - T) T_{\lambda}^{-1/2}} \\
    & \leq 1 + \norm{T_{\lambda}^{-1/2} (T_X - T) T_{\lambda}^{-1/2}} \leq 2.
  \end{align*}
\end{proof}

\subsection{Analytic functional calculus}
\label{subsec:analytic_functional_calculus}

The analytic functional argument is one of the main novelties of this paper.
At a high level, the key point is to convert the operator differences into contour integrals of resolvent terms, so that the problem reduces to obtaining uniform resolvent control along a suitable $\lambda$-dependent contour.
For orientation, the proof in this subsection proceeds in three steps: define the contour $\Gamma_\lambda$, rewrite the operator differences by analytic functional calculus, and then combine this representation with resolvent concentration on $\Gamma_\lambda$.
Let us first recall some basic facts about analytic functional calculus.
We refer, for example, to \citet{simon2015_OperatorTheory} for mathematical details.

\begin{definition}
  Let $A$ be a linear operator on a Banach space $X$.
  The \textit{resolvent set} $\rho(A)$ is given by
  \begin{align*}
    \rho(A) \coloneqq \left\{ \lambda \in \bbC \mid A-\lambda~\text{is invertible} \right\},
  \end{align*}
  and we denote $R_{A}(\lambda) \coloneqq (A-\lambda)^{-1}$.
  The spectrum of $A$ is defined by
  \begin{align*}
    \sigma(A) \coloneqq \bbC \backslash \rho(A).
  \end{align*}
\end{definition}

A simple but key ingredient in the analytic functional calculus is the following \textit{resolvent identity}:
\begin{align}
  \label{eq:ResolventIdentity}
  R_A(\lambda) - R_B(\lambda) = R_A(\lambda) (B-A) R_B(\lambda) = R_B(\lambda) (B-A)R_A(\lambda).
\end{align}

The resolvent allows us to define the value of $f(A)$ analogously to the Cauchy integral formula,
where $A$ is an operator and $f$ is an analytic function.
This is often referred to as analytic functional calculus; see, e.g., \citet[Theorem 2.3.1]{simon2015_OperatorTheory}.

\begin{proposition}[Analytic Functional Calculus]
  \label{prop:func-cal}
  Let $A$ be an operator on a Hilbert space $H$ and $f$ be an analytic function defined on $D_f \subset \bbC$.
  Let $\Gamma$ be a contour contained in $D_f$ surrounding $\sigma(A)$.
  Then,
  \begin{align}
    f(A) = \frac{1}{2\pi i} \oint_{\Gamma} f(z) (z-A)^{-1} \dd z
    = -\frac{1}{2\pi i}\oint_{\Gamma} f(z) R_A(z) \dd z,
  \end{align}
  and it is independent of the choice of $\Gamma$.
\end{proposition}

\begin{remark}
  For a self-adjoint compact operator $A$, we have spectral decomposition
  \begin{align*}
    A = \sum_{i=1}^\infty \lambda_i \ang{e_i,\cdot} e_i,
  \end{align*}
  and $f(A)$ is often defined by
  \begin{align}
    \label{eq:SelfAdjointFuncCal}
    f(A) = \sum_{i=1}^\infty f(\lambda_i) \ang{e_i,\cdot} e_i.
  \end{align}
  In fact, this definition is consistent with the one in \cref{prop:func-cal}.
  We remark that \cref{eq:SelfAdjointFuncCal} is also valid for continuous $f$ and
  an extension to self-adjoint (not necessarily compact) operators is also possible by the spectral theorem~\citep[Section 5]{simon2015_OperatorTheory}.
\end{remark}

Now, let $\Gamma$ be a contour contained in $D_f$ surrounding both $\sigma(A)$ and $\sigma(B)$.
Using \cref{eq:ResolventIdentity}, we get
\begin{align}
  f(A) - f(B)
  &= -\frac{1}{2\pi i}  \oint_{\Gamma} f(z) \left[R_A(z) - R_B(z) \right] \dd z \\
  &= \frac{1}{2\pi i}\oint_{\Gamma} R_B(z) (A-B) R_A(z) f(z)\dd z.
\end{align}

We will use the following spectral mapping theorem to bound some operator norms in the proof,
see \citet[Theorem 5.1.11]{simon2015_OperatorTheory}.
\begin{proposition}[Spectral Mapping Theorem]
  Let $A$ be a bounded self-adjoint operator and $f$ be a continuous function on $\sigma(A)$.
  Then
  \begin{align}
    \sigma(f(A)) =  \left\{ f(\lambda) \mid \lambda \in \sigma(A) \right\}.
  \end{align}
  Consequently, $\norm{f(A)} = \sup_{\lambda \in \sigma(A)} \abs{f(\lambda)} \leq \norm{f}_{\infty}$.
\end{proposition}

Finally, let us define the contour $\Gamma_{\lambda}$ by
\begin{align}
  \label{eq:contour}
  \begin{aligned}
    \Gamma_{\lambda} &= \Gamma_{\lambda,1} \cup \Gamma_{\lambda,2} \cup \Gamma_{\lambda,3} \\
    \Gamma_{\lambda,1} &= \left\{ x \pm (x+\eta) i \in \bbC \mid x \in \left[-\eta, 0\right] \right\} \\
    \Gamma_{\lambda,2} &= \left\{ x \pm (x+\eta) i \in \bbC \mid x \in (0,\kappa^2) \right\} \\
    \Gamma_{\lambda,3} &=
    \left\{ z \in \bbC \mid \abs{z - \kappa^2} = \kappa^2 + \eta, ~ \Re(z) \geq \kappa^2 \right\},
  \end{aligned}
\end{align}
where $\eta = \lambda /2$, see \cref{fig:Contour}.
Then, since $T$ and $T_X$ are positive self-adjoint operators with $\norm{T}, \norm{T_X} \leq \kappa^2$,
we have $\sigma(T), \sigma(T_X) \subset [0,\kappa^2]$.
Therefore, $\Gamma_{\lambda}$ is indeed a contour satisfying the requirement in \cref{prop:func-cal}.
\rev{
  The idea of choosing such a contour is that its distance to the spectral interval $[0,\kappa^2]$ is tuned to the regularization scale $\lambda$, so that \( \sup_{t\in[0,\kappa^2]}\abs{(t+\lambda)/(t+z)} \) remains uniformly bounded along the contour,
  which allows us to replace $T_X$ by $T$ in the resolvent terms with only a constant factor loss, as shown in \cref{prop:ContourSpectralMapping}.
}

\begin{figure}
  \centering
  \begin{tikzpicture}
[decoration={markings,
mark=at position 1cm with {\arrow[line width=1pt]{>}},
mark=at position 5cm with {\arrow[line width=1pt]{>}},
mark=at position 8cm with {\arrow[line width=1pt]{>}}
}
]

  \draw[help lines,->] (-1,0) -- (3,0) coordinate (xaxis);
  \draw[help lines,->] (0,-2) -- (0,2) coordinate (yaxis);

  \path[draw,line width=0.8pt,postaction=decorate] (-0.5,0) node[below left] {$-\lambda/2$} -- (1,-1.5)
  arc (-90:90:1.5)  -- (-0.5,0);
  \draw[{[-]}, line width=0.8pt, dashed] (0,0) -- (1,0) node[below] {$\kappa^2$};

  \node[above] at (0.5,0) {$\sigma(T)$};
  \node[below] at (xaxis) {$\Re$};
  \node[left] at (yaxis) {$\Im$};
  \node[below left] {$0$};
  \node[above] at (1.7,1.3) {$\Gamma_\lambda$};

\end{tikzpicture}   \caption{An illustration of the contour $\Gamma_\lambda$ defined in \cref{eq:contour}.
  The region enclosed by $\Gamma_{\lambda}$ is just $D_\lambda$ in \cref{assu:Filter}.
  The dashed interval $[0,\kappa^2]$ contains the spectrum of $T$ and $T_X$.
  This is the contour along which the resolvent terms in the analytic functional argument are integrated.
  }
  \label{fig:Contour}
\end{figure}

\begin{proposition}
  \label{prop:ContourSpectralMapping}
  When \cref{eq:ConcenRatio} holds,
  there is an absolute constant $C$ such that, for any $z \in \Gamma_\lambda$,
  \begin{align}
    \begin{aligned}
      \norm{T_{\lambda}^{\hf}(T-z)^{-1} T_{\lambda}^{\hf}} &\leq C \\
      \norm{T_{\lambda}^{\hf}(T_X-z)^{-1} T_{\lambda}^{\hf}} &\leq 3C
    \end{aligned}
  \end{align}
\end{proposition}
\begin{proof}
  Using the spectral mapping theorem, for a self-adjoint operator $A$ with $\sigma(A) \subseteq [0,\kappa^2]$ we have
  \begin{align*}
    \norm{A_{\lambda}^{\hf} (A-z)^{-1}A_{\lambda}^{\hf}} =
    \sup_{t \in \sigma(A)} \abs{\frac{t+\lambda}{t-z}}.
  \end{align*}
  Now, when $z = x + (x+\lambda/2)i \in \Gamma_{\lambda,1} \cup \Gamma_{\lambda,2}$, where $x \in [-\lambda/2,\kappa^2]$, we get
  \begin{align*}
    \sup_{t \in \sigma(A)} \abs{\frac{t+\lambda}{t-z}}^2
    \leq \sup_{t \geq 0} \abs{\frac{t+\lambda}{t-z}}^2 =
    \begin{cases}
      \frac{4\lambda^2}{\lambda^2 + 4\lambda x + 8x^2}, & -\frac{1}{2} \lambda \leq x \leq -\frac{1}{2(2+\sqrt {2})} \lambda, \\
      \frac{5\lambda^2 + 12 \lambda x + 8 x^2}{(\lambda+2x)^2}, & x \geq \frac{1}{2(2+\sqrt {2})} \lambda, \\
    \end{cases}
  \end{align*}
  Tedious calculations show that the right-hand side achieves its maximum of $8$ at $x = -\lambda/4$, so
  \begin{align*}
    \sup_{t \in \sigma(A)} \abs{\frac{t+\lambda}{t-z}}^2 \leq 8,\quad z \in \Gamma_{\lambda,1} \cup \Gamma_{\lambda,2}.
  \end{align*}

  When $z \in \Gamma_{\lambda,3}$, we have $\abs{t-z} \geq \kappa^2$ for $t \in \sigma(A)\subseteq [0,\kappa^2]$, so
  \begin{align*}
    \sup_{t \in \sigma(A)}\abs{\frac{t+\lambda}{t-z}}
    \leq \sup_{t \in \sigma(A)}\abs{\frac{t+\lambda}{\kappa^2}} \leq \frac{\lambda + \kappa^2}{\kappa^2} \leq 2.
  \end{align*}
  In summary, we have an absolute constant $C$ such that
  \begin{align*}
    \norm{A_{\lambda}^{\hf} (A-z)^{-1}A_{\lambda}^{\hf}} \leq C.
  \end{align*}

  Consequently, letting $A = T$ yields the first inequality.
  For the second inequality, we note that
  \begin{align*}
    \norm{T_{\lambda}^{\hf}(T_X - z)^{-1} T_{\lambda}^{\hf}}
    &= \norm{T_{\lambda}^{\hf} T_{X\lambda}^{-\hf} \cdot T_{X\lambda}^{\hf} (T_X - z)^{-1} T_{X\lambda}^{\hf} \cdot T_{X\lambda}^{-\hf}T_{\lambda}^{\hf}}  \\
    & \leq \norm{T_{\lambda}^{\hf} T_{X\lambda}^{-\hf}} \cdot
    \norm{T_{X\lambda}^{\hf} (T_X - z)^{-1} T_{X\lambda}^{\hf}} \cdot \norm{T_{X\lambda}^{-\hf}T_{\lambda}^{\hf}} \\
    & \leq 3C,
  \end{align*}
  where we use \cref{eq:ConcenRatio} and the norm bound with $A = T_X$.
\end{proof}

\subsection{The variance term}
\label{subsec:variance_term}

The following theorem greatly improves the results in \citet[Theorem A.10]{li2023_KernelInterpolation} and \citet{zhang2024_OptimalRates}.
Besides the main difference that it considers general spectral algorithms,
it also (1) removes the requirement of Hölder continuity of the kernel function in \citet{li2023_KernelInterpolation};
(2) gives the exact $1+o_{\bbP}(1)$ form with no loss of constant factor compared to \citet{zhang2024_OptimalRates};
and (3) allows a wider range of $\lambda$, leading to a logarithmic lower bound in \cref{cor:Interpolation}.

\begin{theorem}
  \label{thm:Variance}
  Under Assumptions~\ref{assu:EDR},\ref{assu:RegularRKHS} and~\ref{assu:Filter},
  suppose $\lambda = \lambda(n) \to 0$ satisfies
  \begin{align}
    \label{eq:VarianceLambdaCondition}
    \frac{\lambda^{-1/\beta}}{n} (\ln \lambda^{-1})^3 = o(1),
  \end{align}
  then we have
  \begin{align}
    \mathbf{Var}(\lambda) = \left[ 1+o_{\bbP}(1) \right] \frac{\sigma^2}{n} \caN_{2,\varphi}(\lambda).
  \end{align}
  In particular, sufficient conditions for \cref{eq:VarianceLambdaCondition} are
  $\lambda = \Omega(n^{-\theta})$ for some $\theta < \beta$ or $\lambda = \Omega(n^{-\beta} \ln^p n)$ for any $p > 3\beta$.

\end{theorem}
\begin{proof}
  We recall that
  \begin{align*}
    \mathbf{Var}(\lambda) = \frac{\sigma^2}{n} \frac{1}{n} \sum_{i=1}^n  \norm{\reg(T_X)k_{x_i}}^2_{L^2}.

  \end{align*}

  \cref{lem:EffectiveDimEstimationPowerlaw} gives that
  \begin{align}
    \label{eq:N1N2Asymptotics}
    \caN_1(\lambda) \asymp \caN_2(\lambda) \asymp \caN_{2,\varphi}(\lambda) \asymp \lambda^{-1/\beta}.
  \end{align}
  Therefore, the condition \cref{eq:QuantityV} in \cref{lem:Concen} holds as long as $n$ is large enough, since $\lambda = \Omega(n^{-\theta})$ for some $\theta < \beta$.
  Then, applying \cref{lem:Concen}, \cref{lem:VarianceControl1_} and \cref{lem:VarianceControl2}, when $n$ is large enough,
  with probability at least $1-\delta$ we have
  \begin{align*}
    &\quad  \abs{\frac{1}{n} \sum_{i=1}^n  \norm{\reg(T_X)k_{x_i}}^2_{L^2} - \int_{\caX} \norm{\reg(T) k_x}_{L^2}^2 \dd \mu(x)} \\
    & \leq \frac{1}{n} \sum_{i=1}^n \abs{\norm{\reg(T_X)k_{x_i}}^2_{L^2} - \norm{\reg(T)k_{x_i}}^2_{L^2}} \\
    &\qquad + \abs{\frac{1}{n} \sum_{i=1}^n \norm{\reg(T)k_{x_i}}^2_{L^2} - \int_{\caX} \norm{\reg(T) k_x}_{L^2}^2 \dd \mu(x)} \\
    & \leq CM   \left( \sqrt {v\caN_1(\lambda)}\ln \lambda^{-1} + \sqrt {\caN_{2,\varphi}(\lambda)} \right) \\
    &\qquad \cdot \sqrt {v\caN_1(\lambda)} \ln \lambda^{-1} \\
    &\qquad + M \caN_{2,\varphi}(\lambda) \sqrt {\frac{2}{n} \ln \frac{2}{\delta}} \\
    & = o\left(\caN_{2,\varphi}(\lambda)\right),
  \end{align*}
  where for the last estimate, we recall that $v$ is given by \cref{eq:QuantityV},
  so by \cref{eq:VarianceLambdaCondition}, we get
  \begin{align*}
    \left[ \sqrt {v\caN_1(\lambda)} \ln \lambda^{-1} \right]^2
    &= \frac{M \caN_1^2(\lambda)}{n} \ln \frac{\caN_1(\lambda)}{\delta}  (\ln \lambda^{-1})^2 \\
    &\leq C\ln \frac{1}{\delta} \cdot \frac{\lambda^{-1/\beta}}{n}  (\ln \lambda^{-1})^3 \\
    &\qquad \cdot \lambda^{-1/\beta}
    = o\left(\lambda^{-1/\beta}\right)
    = o\left(\caN_{2,\varphi}(\lambda)\right).
  \end{align*}

  Finally, using Mercer's expansion, we find that
  \begin{align*}
    \norm{\reg(T) k_x}_{L^2}^2 &= \norm{\reg(T) \sum_{m = 1}^{\infty} \mu_m \sum_{l = 1}^{d_m} \overline{e_{m,l}(x)}  e_{m,l}}_{L^2}^2 \\
    &= \norm{\sum_{m = 1}^{\infty} \reg(\mu_m)\mu_m \sum_{l = 1}^{d_m} \overline{e_{m,l}(x)}  e_{m,l}}_{L^2}^2 \\
    &= \sum_{m = 1}^{\infty} \left( \reg(\mu_m) \mu_m  \right)^2 \sum_{l = 1}^{d_m} \abs{e_{m,l}(x)}^2,
  \end{align*}
  and thus the deterministic term can be written as
  \begin{align*}
    \int_{\caX} \norm{\reg(T) k_x}_{L^2}^2 \dd \mu(x) &=
    \int_{\caX} \left[ \sum_{m = 1}^{\infty} \left( \reg(\mu_m) \mu_m  \right)^2 \sum_{l = 1}^{d_m} \abs{e_{m,l}(x)}^2 \right] \dd \mu(x) \\
    &= \sum_{m = 1}^{\infty} \left( \reg(\mu_m) \mu_m  \right)^2 = \caN_{2,\varphi}(\lambda).
  \end{align*}
\end{proof}

\begin{lemma}
  \label{lem:VarianceMonotone}
  The variance term $\mathbf{Var}(\lambda)$ increases as $\lambda$ decreases, i.e., for any $\lambda_1 \leq \lambda_2$,
  we have $\mathbf{Var}(\lambda_1) \geq \mathbf{Var}(\lambda_2)$.
\end{lemma}
\begin{proof}
  Let us define the kernel matrix $K = \frac{1}{n}\big(k(x_i,x_j)\big)_{n\times n}$.

  Then, it is easy to show that the representation matrix of $T_X$ on the set $\{k_{x_i}\}_{i=1}^n$ is given by $K$
  (see, for example, \citet[Section A.1]{li2023_KernelInterpolation}).
  Consequently, denoting a column vector $\K(X,\cdot) = (k_{x_1},\dots,k_{x_n})^T$, we have
  \begin{align*}
    & \reg(T_X) \K(X,\cdot) = \reg(K) \K(X,\cdot),

  \end{align*}
  where the action of $\reg(T_X)$ on the left-hand side is element-wise.

  Then,
  \begin{align*}
    \mathbf{Var}(\lambda)
    &= \frac{\sigma^2}{n^2}  \sum_{i=1}^n  \norm{\reg(T_X)k_{x_i}}^2_{L^2} \\
    &= \frac{\sigma^2}{n^2}   \sum_{i=1}^n \int_{\caX} \abs{(\reg(T_X)k_{x_i})(x)}^2 \dd \mu(x) \\
    &= \frac{\sigma^2}{n^2}  \int_{\caX} \sum_{i=1}^n  \abs{(\reg(T_X)k_{x_i})(x)}^2 \dd \mu(x) \\
    &= \frac{\sigma^2}{n^2}  \int_{\caX}  \norm{(\reg(T_X)\K(X,\cdot))(x)}_{\R^n}^2 \dd \mu(x) \\
    &= \frac{\sigma^2}{n^2}  \int_{\caX}  \norm{(\reg(K) \K(X,\cdot))(x)}_{\R^n}^2 \dd \mu(x) \\
    &= \frac{\sigma^2}{n^2}  \int_{\caX}  \norm{\reg(K) \K(X,x)}_{\R^n}^2 \dd \mu(x) \\
    &= \frac{\sigma^2}{n^2}\int_{\caX} \left[ \K(X,x) \right]^H \reg^2(K) \K(X,x) \dd \mu(x),
  \end{align*}
  where $\left[ \K(X,x) \right]^H$ is the conjugate transpose of $\K(X,x)$.
  Moreover, the property (i) of the filter function implies that $\reg(z)$ increases as $\lambda$ decreases.
  Therefore, we get
  $\varphi_{\lambda_1}^2(K) \succeq  \varphi_{\lambda_2}^2(K)$ and the result follows.
\end{proof}

\begin{corollary}
  \label{cor:Interpolation}
  When $\lambda = \lambda(n) = O(n^{-\beta})$, we have
  \begin{align}
    \mathbf{Var}(\lambda) = \Omega_{\bbP}\left( (\ln n)^{-4} \sigma^2 \right).
  \end{align}
\end{corollary}
\begin{proof}
  Let us choose $\tilde{\lambda} = n^{-\beta} (\ln n)^{4\beta}$, then we have $\lambda \leq \tilde{\lambda}$ when $n$ is large enough.
  Using \cref{lem:VarianceMonotone}, we get $\mathbf{Var}(\lambda) \geq \mathbf{Var}(\tilde{\lambda})$.
  Moreover, the choice of $\tilde{\lambda} $ satisfies the condition \cref{eq:VarianceLambdaCondition}, so applying \cref{thm:Variance} yields
  \begin{align*}
    \mathbf{Var}(\tilde{\lambda}) = \left[ 1+o_{\bbP}(1) \right] \frac{\sigma^2}{n} \caN_{2,\varphi}(\tilde{\lambda})
    = \Omega_{\bbP}\left(  \frac{\sigma^2}{n}\tilde{\lambda}^{-1/\beta} \right)
    = \Omega_{\bbP}\left( \sigma^2 (\ln n)^{-4} \right).
  \end{align*}
\end{proof}

\begin{lemma}
  \label{lem:VarianceControl2}

  With probability at least $1-\delta$, we have
  \begin{align}
    \label{eq:VarianceControl2}
    \abs{\frac{1}{n}\sum_{i=1}^n \norm{\reg(T) k_{x_i}}_{L^2}^2 - \int_{\caX} \norm{\reg(T) k_x}_{L^2}^2 \dd \mu(x)}
    \leq M \caN_{2,\varphi}(\lambda) \sqrt {\frac{2}{n} \ln \frac{2}{\delta}}.
  \end{align}
\end{lemma}
\begin{proof}
  Let $\xi(x) = \norm{\reg(T) k_x}_{L^2}^2$ and $\xi_i = \xi(x_i)$.
  Then, they are i.i.d.\ random variables, and
  \begin{align*}
    \abs{\xi} = \norm{\reg(T) k_x}_{L^2}^2  \leq M \caN_{2,\varphi}(\lambda)
  \end{align*}
  from \cref{eq:RegPhiKx}.
  Then, \cref{lem:HoeffdingInequality} yields the desired result.
\end{proof}

\begin{lemma}
  \label{lem:VarianceControl1_}
  Under \cref{assu:RegularRKHS} and \cref{assu:Filter}, when \cref{eq:ConcenU} and \cref{eq:ConcenRatio} hold, we have
  \begin{equation}
    \label{eq:VarianceControl1_}
    \begin{multlined}
      \sup_{x \in \caX} \abs{\norm{\reg(T_X) k_x}_{L^2}^2 - \norm{\reg(T) k_x}_{L^2}^2} \\
      \leq CM   \left( \sqrt {v\caN_1(\lambda)}\ln \lambda^{-1} + \sqrt {\caN_{2,\varphi}(\lambda)} \right) \\
      \qquad \cdot \sqrt {v\caN_1(\lambda)}\ln \lambda^{-1}.
    \end{multlined}
  \end{equation}
\end{lemma}
\begin{proof}
  We start with
  \begin{align*}
    D = \abs{\norm{\reg(T_X) k_x}_{L^2} - \norm{\reg(T) k_x}_{L^2}}
    \leq \norm{T^{\hf}\left[ \reg(T)-\reg(T_X) \right] k_x}_{\caH}.
  \end{align*}
  Using operator calculus, we get
  \begin{align*}
    & \quad T^{\hf}\left[ \reg(T)-\reg(T_X) \right] k_x \\
    &= T^{\hf} \left[ \frac{1}{2\pi i}\oint_{\Gamma_{\lambda}} R_{T_X}(z) (T-T_X) R_T(z)\reg(z) \dd z \right]  k_x  \\
    &= \frac{1}{2\pi i} \oint_{\Gamma_{\lambda}} T^{\hf} (T_X-z)^{-1} (T-T_X) (T-z)^{-1} k_x \reg(z) \dd z \\
    &= \frac{1}{2\pi i} \oint_{\Gamma_{\lambda}} T^\hf T_{\lambda}^{-\hf}
    \cdot T_{\lambda}^{\hf} (T_X -z)^{-1}T_{\lambda}^{\hf} \\
    &\qquad \cdot T_{\lambda}^{-\hf}(T-T_X)T_{\lambda}^{-\hf}
    \cdot T_{\lambda}^{\hf} (T-z)^{-1} T_{\lambda}^{\hf} \\
    &\qquad \cdot T_{\lambda}^{-\hf}k_x \reg(z) \dd z.
  \end{align*}
  Therefore, taking the norms yields
  \begin{align*}
    D
    & \leq C \norm{T^\hf T_{\lambda}^{-\hf} }
    \norm{T_{\lambda}^{\hf} (T_X-z)^{-1}T_{\lambda}^{\hf}}
    \norm{T_{\lambda}^{-\hf}(T-T_X)T_{\lambda}^{-\hf}} \\
    &\qquad \cdot \norm{T_{\lambda}^{\hf} (T-z)^{-1} T_{\lambda}^{\hf}}
    \norm{T_{\lambda}^{-\hf}k_x}_{\caH} \\
    &\qquad \cdot \oint_{\Gamma_{\lambda}} \abs{\reg(z) \dd z} \\
    & \leq C \sqrt{M v \caN_1(\lambda)}
    \oint_{\Gamma_{\lambda}} \abs{\reg(z) \dd z},
  \end{align*}
  where in the second estimation, we use respectively for part (1): operator calculus, (2,4): \cref{prop:ContourSpectralMapping},
  (3): estimation \cref{eq:ConcenU} and (5): estimation \cref{eq:RegKx}.
  With \cref{assu:Filter}, we get
  \begin{align*}
    \oint_{\Gamma_{\lambda}} \abs{\reg(z) \dd z} \leq C \oint_{\Gamma_{\lambda}} \frac{1}{\abs{z+\lambda}} \abs{\dd z}.
  \end{align*}
  Now we focus on the latter integral.
  For $z \in \Gamma_{\lambda,1}$, we have $\abs{z+\lambda} \geq \lambda/(2\sqrt{2})$ and thus
  \begin{align*}
    \int_{\Gamma_{\lambda,1}} \frac{1}{\abs{z+\lambda}} \abs{\dd z}
    \leq 2\sqrt{2}\lambda^{-1} \abs{\Gamma_{\lambda,1}} \leq C,
  \end{align*}
  where we notice that $\abs{\Gamma_{\lambda,1}} \leq C \lambda$.
  For $\Gamma_{\lambda,2}$, we have
  \begin{align*}
    \int_{\Gamma_{\lambda,2}} \frac{1}{\abs{z+\lambda}} \abs{\dd z} &= 2 \int_{0}^{\kappa^2} \frac{1}{\abs{x+(x+\lambda/2)i+\lambda}} \sqrt {2}\dd x \\
    & \leq C \int_{0}^{\kappa^2} \frac{1}{x+\lambda} \dd x \\
    & \leq C \ln \lambda^{-1}.
  \end{align*}
  For $z \in \Gamma_{\lambda,3}$, we have $\abs{z + \lambda} \geq \kappa^2$ and thus
  \begin{align*}
    \int_{\Gamma_{\lambda,3}}\frac{1}{\abs{z+\lambda}} \abs{\dd z}  \leq \frac{1}{\kappa^2}\abs{\Gamma_{\lambda,3}} \leq C.
  \end{align*}
  Therefore, we get
  \begin{align}
    \label{eq:ContourIntegralAbs}
    \oint_{\Gamma_{\lambda}} \frac{1}{\abs{z+\lambda}} \abs{\dd z} \leq C \ln \lambda^{-1},
  \end{align}
  and thus
  \begin{align*}
    D = \abs{\norm{\reg(T_X) k_x}_{L^2} - \norm{\reg(T) k_x}_{L^2}} \leq C \sqrt{M v \caN_1(\lambda)} \ln \lambda^{-1}.
  \end{align*}

  Then, combining with the second estimation in \cref{eq:RegPhiKx}, we have
  \begin{align*}
    \norm{\reg(T_X) k_x}_{L^2} + \norm{\reg(T)k_x}_{L^2}
    &\leq 2\norm{\reg(T) k_x}_{L^2} + D \\
    &\leq C \sqrt{M}\left(
                      \sqrt {\caN_{2,\varphi}(\lambda)}  + \sqrt {v\caN_1(\lambda)} \ln \lambda^{-1}
    \right).
  \end{align*}
  Finally,
  \begin{align*}
    &\quad \abs{\norm{\reg(T_X) k_x}_{L^2}^2 - \norm{\reg(T) k_x}_{L^2}^2} \\
    &= \abs{\norm{\reg(T_X) k_x}_{L^2} - \norm{\reg(T) k_x}_{L^2}}
    \cdot \left( \norm{\reg(T_X) k_x}_{L^2} + \norm{\reg(T) k_x}_{L^2} \right) \\
    &\leq CM  \left( \sqrt {\caN_{2,\varphi}(\lambda)}  + \sqrt {v\caN_1(\lambda)} \ln \lambda^{-1} \right)
    \cdot \sqrt {v\caN_1(\lambda)} \ln \lambda^{-1}.
  \end{align*}

\end{proof}

\subsection{The bias term}
\label{subsec:bias_term}

\begin{theorem}
  \label{thm:Bias}
  Let $\lambda = \lambda(n)$ satisfy $\lambda=\Omega(n^{-\theta})$ for some $\theta<\beta$.
  Under Assumptions~\ref{assu:EDR},\ref{assu:RegularRKHS},\ref{assu:Filter} and~\ref{assu:Source}, we have
  \begin{equation}
    \label{eq:BiasTermApprox}
    \mathbf{Bias}^2(\lambda)=\caR_{\varphi}^2(\lambda;\fstar)+ o_{\bbP}\left( \caR_{\varphi}^2(\lambda;\fstar) + \frac{1}{n} \caN_{2,\varphi}(\lambda) \right).
  \end{equation}
  More precisely, letting $\tilde{s} = \min(s,2\tau_{\max})$,
  if $\tilde{s} \leq 2$, or $\tilde{s} > 2$ with further $n^{-1}\lambda^{-(\beta^{-1}+\tilde{s}-2)} (\ln n)^3 = o(1)$, we have
  \begin{equation}
    \mathbf{Bias}^2(\lambda)=(1+o_{\bbP}(1))\caR_{\varphi}^2(\lambda;\fstar).
  \end{equation}
\end{theorem}

\begin{proof}[Proof of \cref{thm:Bias}]
  First, we can apply \cref{cor:Concen} so that \cref{eq:ConcenU} and \cref{eq:ConcenRatio} hold.
  We recall that
  \begin{align*}
    g^* = T f^*,\quad \flam = \reg(T)T f^* = \reg(T)g^*.
  \end{align*}
  As mentioned in \cref{subsec:proof_sketch}, the bias term is defined as
  \begin{equation*}
    \begin{aligned}
      \mathbf{Bias}^2(\lambda)&=\norm{\tilde{f}_\lambda-f^*}_{L^2}^2
      =\norm{\flam-f^*+\tilde{f}_\lambda-\flam}_{L^2}^2.
    \end{aligned}
  \end{equation*}
  Hence,
  \begin{equation*}
    \norm{\flam-f^*}_{L^2}-\norm{\tilde{f}_\lambda-\flam}_{L^2}\leq\mathbf{Bias}(\lambda)\leq\norm{\flam-f^*}_{L^2}+\norm{\tilde{f}_\lambda-\flam}_{L^2},
  \end{equation*}
  where $\norm{\flam-f^*}_{L^2}=\caR_{\varphi}(\lambda;\fstar)$ is the main term defined in \cref{eq:BiasMainTerm}.
  As for the error term, we make the decomposition
  \begin{equation}
    \label{eq:bias_error_decomposition}
    \begin{aligned}
      \tilde{f}_\lambda-\flam &=\reg(T_X)\gtl-(\rem(T_X)+\reg(T_X)T_X)\flam\\
      &= \reg(T_X)(\gtl-T_X \flam)-\rem(T_X)T\reg(T)f^*\\
      &= \reg(T_X)(\gtl-T_X \flam)  - \reg(T_X)\rem(T) g^* \\
      &\qquad + \reg(T_X)\rem(T) g^* - \rem(T_X)T\reg(T)f^* \\
      &= \reg(T_X)\left[ \gtl-T_X\flam- \rem(T)g^*  \right] \\
      &\qquad + \left[ \reg(T_X) \rem(T) T f^* - \rem(T_X)T\reg(T)f^* \right] \\
      &=\reg(T_X)(\gtl-T_X\flam-g^*+T\flam) \\
      &\qquad +(\reg(T_X)T\rem(T)-\rem(T_X)T\reg(T))f^*.
    \end{aligned}
  \end{equation}
  For the first term in \cref{eq:bias_error_decomposition},
  \begin{align*}
    &\quad\norm{\reg(T_X)(\gtl-T_X\flam-g^*+T\flam)}_{L^2}\\
    &=\norm{T^{\frac{1}{2}}\reg(T_X)(\gtl-T_X\flam-g^*+T\flam)}_{\caH}\\
    & \leq \norm{T^{\frac{1}{2}}T_\lambda^{-\frac{1}{2}}}\cdot\norm{T_\lambda^\frac{1}{2}\reg(T_X)T_\lambda^\frac{1}{2}}\cdot\norm{T_\lambda^{-\hf}\left[\left(\gtl-T_X \flam \right)-\left(g^* - T \flam \right)\right]}_{\caH}\\
    & \stackrel{(i)}{\leq} 1 \cdot
    E \norm{T_\lambda^\frac{1}{2}T_{X\lambda}^{-1}T_\lambda^\frac{1}{2}} \cdot \norm{T_\lambda^{-\hf}\left[\left(\gtl-T_X \flam \right)-\left(g^* - T \flam \right)\right]}_{\caH}\\
    & \leq C \norm{T_\lambda^{-\hf}\left[\left(\gtl-T_X \flam \right)-\left(g^* - T \flam \right)\right]}_{\caH},
  \end{align*}
  where the second control in (i) comes from \cref{eq:Filter_Reg} and the last one can be derived from \cref{eq:ConcenRatio}.
  Employing \cref{prop:ResidualTermLowerBound}, we also have
  \begin{equation}
    \label{eq:BiasTermLowerBound}
    \caR_{\varphi}(\lambda;\fstar)=\Omega(\lambda^{\tilde{s}/2}),
  \end{equation}
  where we denote $\tilde{s} = \min(s,2\tau_{\max})$.
  Hence, owing to \cref{lem:RegApprox} with $t$ sufficiently close to $s$, we have
  \begin{equation}
    \label{eq:BiasTermError1}
    \norm{T_\lambda^{\hf}(\gtl-T_X\flam-g^*+T\flam)}_{L^2}=o_{\bbP}\left(\caR_{\varphi}(\lambda;\fstar)\right).
  \end{equation}

  For the second term in \cref{eq:bias_error_decomposition},
  since $\lambda=\Omega(n^{-\theta})$ for some $\theta<\beta$,
  as discussed \cref{eq:N1N2Asymptotics} in the proof of variance term, we have
  \begin{align*}
    v=O(n^{-1}\caN_1(\lambda)\ln \lambda^{-1})=O(n^{-1}\lambda^{-\frac{1}{\beta}}\ln \lambda^{-1})= o(1),
  \end{align*}
  so the condition in \cref{lem:Concen} is satisfied.
  Then, combining \cref{lem:Concen} and \cref{lem:bias_control},
  for any fixed $t$ satisfying $t<s$ and $t\leq 2$,
  \begin{align*}
    \norm{(\reg(T_X)T\rem(T)-\rem(T_X)T\reg(T))f^*}_{L^2}
    &= O_{\bbP} \left(
                  \caR_{\varphi}(\lambda;f^*) + \norm{\fstar}_{[\caH]^t}\lambda^\frac{t}{2}
    \right) \\
    &\qquad \cdot \sqrt{v}\ln\lambda^{-1}.
  \end{align*}
  Moreover, we also have $\sqrt{v}\ln\lambda^{-1} = o(1)$ so
  \begin{align*}
    \caR_{\varphi}(\lambda;f^*) \sqrt{v}\ln\lambda^{-1} = o(\caR_{\varphi}(\lambda;f^*)).
  \end{align*}
  For the last term, we notice
  \begin{align*}
    \left( \lambda^\frac{t}{2}\sqrt{v}\ln\lambda^{-1} \right)^2
    = O\left( \lambda^{t} n^{-1}\lambda^{-\frac{1}{\beta}}(\ln \lambda^{-1})^{3} \right).
  \end{align*}
  Let us consider:
  \begin{itemize}
    \item \textbf{Case 1:} Using \cref{eq:BiasTermLowerBound}, if for some $t < s$ and $t \leq 2$,
    \begin{align}
      \label{eq:BiasLambdaCondition1}
      n^{-1}\lambda^{-(\beta^{-1}+\tilde{s}-t)} (\ln \lambda^{-1})^3 = O\left(n^{-1}\lambda^{-(\beta^{-1}+\tilde{s}-t)} (\ln n)^3\right) = o(1),
    \end{align}
    we have $\lambda^\frac{t}{2}\sqrt{v}\ln\lambda^{-1} = o(\caR_{\varphi}(\lambda;f^*))$.
    \item \textbf{Case 2:} Using \cref{eq:N1N2Asymptotics}, if $\lambda^{t}(\ln \lambda^{-1})^{3} = O(\lambda^{t}(\ln n)^3) = o(1)$, we have
    $\lambda^\frac{t}{2}\sqrt{v}\ln\lambda^{-1} = o\left( \frac{1}{n}\caN_{2,\varphi}(\lambda) \right)^{1/2}.$
  \end{itemize}

  Now, if $\tilde{s} \leq 2$, then since $\lambda=\Omega(n^{-\theta})$ for some $\theta<\beta$,
  \cref{eq:BiasLambdaCondition1} can always be satisfied by choosing $t$ sufficiently close to $s$, namely
  $t > \tilde{s} - \left( \frac{1}{\theta} - \frac{1}{\beta} \right)$, so we always have the result in case 1.

  On the other hand, if $\tilde{s} > 2$, we fix $t = 2$ and fix some $0 < \theta_0 < (\beta^{-1} + \tilde{s} - 2)^{-1}$.
  Then, when $\lambda \geq n^{-\theta_0}$, case 1 applies; and when $\lambda \leq n^{-\theta_0}$, case 2 applies.
  In summary, we always have
  \begin{align*}
    \lambda^\frac{t}{2}\sqrt{v}\ln\lambda^{-1} = o\left( \caR_{\varphi}^2(\lambda;f^*) + \frac{1}{n}\caN_{2,\varphi}(\lambda)\right)^{1/2}.
  \end{align*}

  Consequently, we have shown that
  \begin{align*}
    \norm{[\reg(T_X)T\rem(T)-\rem(T_X)T\reg(T)]f^*}_{L^2} = o_{\bbP}\left( \caR_{\varphi}^2(\lambda;f^*) + \frac{1}{n}\caN_{2,\varphi}(\lambda)\right)^{1/2}.
  \end{align*}
  Combining it with \cref{eq:BiasTermError1}, we prove that the error term $\norm{\tilde{f}_\lambda-\flam}_{L^2}$ is also of this order
  and \cref{eq:BiasTermApprox} follows.

\end{proof}

The following lemma is a control of an approximation error in the bias term,
which is similar to the combination of Lemma A.5 and Lemma A.10 in \citet{li2023_AsymptoticLearning},
but we consider general spectral algorithms here.
Moreover, we also apply the techniques in \citet{zhang2023_OptimalityMisspecifieda} to deal with the misspecified case.
The proof is deferred to the appendix.

\begin{lemma}
  \label{lem:RegApprox}
  Let Assumptions~\ref{assu:EDR} and~\ref{assu:RegularRKHS} hold,
  $\fstar \in [\caH]^t$ and $\reg$ be a filter function with qualification $\tau$.
  Suppose $\lambda = \Omega(n^{-\theta})$ for some $\theta < \beta$.
  Then, there exists some $\ep > 0$ (depending on $\theta$) such that
  \begin{align}
    \label{eq:RegApprox}
    \norm{T_\lambda^{-\hf} \left[\left(\gtl-T_X \flam \right)-\left(g^* - T \flam \right)\right]}_\caH
    = O_{\bbP}\left( n^{-\ep} \lambda^{\tilde{t}/2} \right).
  \end{align}
  where $g^* = T f^*$, $\flam = T\reg(T) f^*$ and $\tilde{t} = \min(t,2\tau)$.
\end{lemma}

The next lemma deals with the interaction term in \cref{eq:bias_error_decomposition},
where we apply the analytic functional argument.

\begin{lemma}
  \label{lem:bias_control}  Under Assumptions~\ref{assu:RegularRKHS},\ref{assu:Filter},
  assume that $\fstar \in [\caH]^t$,
  when \cref{eq:ConcenU} and \cref{eq:ConcenRatio} hold, we have
  \begin{equation*}
    \begin{aligned}
      \norm{(\reg(T_X)T\rem(T)-\rem(T_X)T\reg(T))f^*}_{L^2}
      &\leq C \left(
                \caR_{\varphi}(\lambda;f^*) + \norm{\fstar}_{[\caH]^{\tilde{t}}}\lambda^{\tilde{t}/2}
      \right) \\
      &\qquad \cdot \sqrt{v}\ln\lambda^{-1},
    \end{aligned}
  \end{equation*}
  where $\tilde{t} = \min(t,2)$.
\end{lemma}
\begin{proof}
  First, let us decompose
  \begin{equation}
    \label{eq:decomposition2}
    \begin{multlined}
      \norm{(\reg(T_X)T\rem(T)-\rem(T_X)T\reg(T))f^*}_{L^2} \\
      = \norm{T^\frac{1}{2}(\reg(T_X)T\rem(T)-\rem(T_X)T\reg(T))f^*}_\caH \\
      \leq \norm{T^\frac{1}{2}(\reg(T_X)T\rem(T)-\rem(T)T\reg(T))f^*}_\caH \\
      \qquad + \norm{T^\frac{1}{2}(\rem(T_X)T\reg(T)-\rem(T)T\reg(T))f^*}_\caH.
    \end{multlined}
  \end{equation}

  For the second term in \cref{eq:decomposition2},
  we use a similar argument to that in \cref{lem:VarianceControl1_}.
  With \cref{prop:func-cal} on $\Gamma_\lambda$ defined as \cref{eq:contour}, we have
  \begin{align*}
    &\quad T^\frac{1}{2}(\rem(T_X)T\reg(T)-\rem(T)T\reg(T))f^*\\
    &= T^{\hf}  \left[ \frac{1}{2\pi i}\oint_{\Gamma_{\lambda}} R_{T_X}(z) (T-T_X) R_T(z) \rem(z) \dd z \right] T\reg(T)f^* \\
    & = \frac{1}{2\pi i}\oint_{\Gamma_{\lambda}} T^{\hf} (T_X - z)^{-1}(T-T_X) (T-z)^{-1}  \rem(z) T\reg(T)f^*\dd z\\
    & =  \frac{1}{2\pi i} \int_{\Gamma_{\lambda}}T^\hf T_{\lambda}^{-\hf} \cdot T_{\lambda}^{\hf} (T_{X}-z)^{-1}T_{\lambda}^{\hf}
    \cdot  T_{\lambda}^{-\hf}(T-T_X)T_{\lambda}^{-\hf} \\
    & \qquad\cdot T_{\lambda}^{\hf} (T-z)^{-1} T_{\lambda}^{\hf} \cdot T_{\lambda}^{-\hf}T\reg(T)f^* \rem(z) \dd z.
  \end{align*}
  Hence,
  \begin{align*}
    &\quad \norm{T^\frac{1}{2}(\rem(T_X)T\reg(T)-\rem(T)T\reg(T))f^*}_{\caH}\\
    &\leq C \int_{\Gamma_{\lambda}}  \norm{T^\hf T_{\lambda}^{-\hf}} \cdot \norm{T_{\lambda}^{\hf} (T_{X}-z)^{-1}T_{\lambda}^{\hf}}
    \cdot \norm{T_{\lambda}^{-\hf}(T-T_X)T_{\lambda}^{-\hf}} \\
    & \qquad \cdot \norm{T_{\lambda}^{\hf} (T-z)^{-1} T_{\lambda}^{\hf}} \cdot
    \norm{ T_{\lambda}^{-\hf}T\reg(T)f^*}_{\caH} \abs{\rem(z) \dd z} \\
    &\stackrel{(a)}{ \leq } C \cdot 1 \cdot C \cdot  \sqrt {v} \cdot C \cdot \norm{ T_{\lambda}^{-\hf}T\reg(T)f^*}_{\caH} \cdot \lambda \int_{\Gamma_{\lambda}} \abs{\frac{1}{z+\lambda} \dd z} \\
    &\stackrel{(b)}{ \leq } C \sqrt {v} \norm{ T_{\lambda}^{-\hf}T\reg(T)f^*}_{\caH} \lambda \ln\lambda^{-1},
  \end{align*}
  where in (a), we use (1) operator calculus, (2,4) \cref{prop:ContourSpectralMapping},
  and (3) estimation \cref{eq:ConcenU} and (6) condition (C2) in \cref{assu:Filter} for the corresponding factors, respectively,
  and in (b) we apply \cref{eq:ContourIntegralAbs} for the last term.

  Let $\tilde{t} = \min(t,2)$.
  Since $f^*\in [\caH]^t$, we also have $f^* \in [\caH]^{\tilde{t}}$,
  so we can write $f^*=T^{\tilde{t}} h$ for some $h\in L^2$ with $\norm{h}_{L^2} = \norm{\fstar}_{[\caH]^{\tilde{t}}}  $.
  This yields
  \begin{equation*}
    \begin{aligned}
      \norm{ T_{\lambda}^{-\hf}T\reg(T)f^*}_{\caH}&=\norm{T^{-\hf} T_{\lambda}^{-\hf}T\reg(T)T^{\frac{\tilde{t}}{2}}h}_{L^2}\\
      &\leq \norm{T^\hf T_{\lambda}^{-\hf}}\cdot\norm{T^{\frac{\tilde{t}}{2}}\reg(T)}\cdot\norm{h}_{L^2}\\
      &\leq C \norm{\fstar}_{[\caH]^{\tilde{t}}} \lambda^{\frac{\tilde{t}}{2}-1},
    \end{aligned}
  \end{equation*}
  where the last inequality comes from \cref{lem:Filter_MoreControl}.
  Consequently, we have
  \begin{equation}
    \label{eq:bias_control}
    \norm{T^\frac{1}{2}(\rem(T_X)T\reg(T)-\rem(T)T\reg(T))f^*}_{\caH}
    \leq C \norm{\fstar}_{[\caH]^{\tilde{t}}}  \sqrt {v} \lambda^\frac{\tilde{t}}{2} \ln\lambda^{-1}.
  \end{equation}

  For the first term in \cref{eq:decomposition2}, we still employ the analytic functional argument:
  \begin{equation*}
    \begin{aligned}
      &\quad T^\frac{1}{2}(\reg(T_X)T\rem(T)-\rem(T)T\reg(T))f^*\\
      &=T^\frac{1}{2}(\reg(T_X)-\reg(T))T\rem(T)f^*\\
      &=\frac{1}{2\pi i}\oint_{\Gamma_{\lambda}} T^\frac{1}{2}(T_X-z)^{-1}(T_X-T)(T-z)^{-1}\reg(z)T\rem(T)f^*\dd z\\
      &= \frac{1}{2\pi i} \oint_{\Gamma_{\lambda}}T^\hf T_{\lambda}^{-\hf} \cdot T_{\lambda}^{\hf} (T_{X}-z)^{-1}T_{\lambda}^{\hf}
      \cdot T_{\lambda}^{-\hf}(T-T_X)T_{\lambda}^{-\hf} \\
      & \quad \cdot T_{\lambda}^{\hf} (T-z)^{-1}  T_\lambda^{\hf}\cdot T_{\lambda}^{-\hf}T^{\hf}  \cdot
      T^\hf \rem(T) f^* \reg(z) \dd z.
    \end{aligned}
  \end{equation*}
  Therefore,
  \begin{align*}
    &\quad\lVert T^\frac{1}{2}(\reg(T_X)T\rem(T)-\rem(T)T\reg(T))f^*\rVert_\caH\\
    &\leq \oint_{\Gamma_{\lambda}}\norm{T^\hf T_{\lambda}^{-\hf}} \cdot \norm{T_{\lambda}^{\hf} (T_X -z)^{-1}T_{\lambda}^{\hf}}
    \cdot \norm{T_{\lambda}^{-\hf}(T-T_X)T_{\lambda}^{-\hf}}\cdot   \\
    &\quad \cdot   \norm{T_{\lambda}^{\hf} (T-z)^{-1} T_{\lambda}^{\hf}}  \norm{T_{\lambda}^{-\hf}T^{\hf}} \cdot \norm{ T^\hf \rem(T) f^*}_{\caH} \abs{\reg(z) \dd z}\\
    &\leq C \sqrt{v}\caR_{\varphi}(\lambda;f^*)\oint_{\Gamma_{\lambda}}\abs{\reg(z) \dd z}\\
    &\leq C \sqrt{v}\caR_{\varphi}(\lambda;f^*)\ln\lambda^{-1}.
  \end{align*}
  where the last control holds owing to condition (C1) in \cref{assu:Filter}.
  Consequently, plugging the previous control and \cref{eq:bias_control} into \cref{eq:decomposition2} yields the desired result.

\end{proof}

\section{Conclusion}
\label{sec:conclusion}

In this paper,
we rigorously established a full characterization of the generalization error curves for a large class of analytic spectral algorithms,
providing an exact and complete picture of the generalization errors of these kernel methods.
Our result shows the interplay between the kernel, the regression function, the noise level, and the choice of the regularization parameter.
In particular, it shows a clear U-shaped bias-variance trade-off curve with respect to the regularization parameter.
As applications, it recovers the minimax optimal rates, shows poor generalization in the interpolating regime,
and also reveals a high-order saturation effect.
These results greatly improve our understanding of the generalization behavior of spectral algorithms.

\rev{
  We first comment on the power-law assumptions on the eigenvalue decay (\cref{assu:EDR}) and the source condition (\cref{assu:Source}).
  We do not view the specific power-law form as essential.
  They are mainly used to express the key deterministic terms in a transparent way, such as the bias term $\caR_{\varphi}^2(\lambda;\fstar)$, the variance term $\caN_{2,\varphi}(\lambda)$, and the associated remainder conditions.
  We expect that one can replace these explicit power-law assumptions by direct requirements on such key terms, thereby covering more general decays and source behaviors, but this would introduce additional case-by-case notation and technical estimates and would obscure the main message of the paper.
  We therefore keep the present power-law formulation for clarity.
}

\rev{
  The regular RKHS condition (\cref{assu:RegularRKHS}) plays a different role.

  It provides the sharp control needed to keep the remainder terms infinitesimal.
  In particular, it allows the regularization parameter to go down to the critical scale $\lambda \asymp n^{-\beta}$, which is what leads to the nearly constant lower bound in the interpolating regime.
  More generally, if one only has an embedding index (see \cref{eq:EMB}) $\alpha \in [1/\beta,1]$, then we expect the same proof strategy to yield analogous exact decompositions on the range $\lambda = \Omega(n^{-\theta})$ for $\theta < 1/\alpha$, that is, down to nearly the scale $n^{-1/\alpha}$, but no further.
  In this sense, one should still obtain an exact characterization on a nontrivial range of $\lambda$, although narrower than in the regular RKHS case and with stronger side conditions.
}

\rev{
  The analyticity assumption (\cref{assu:Filter}) is a key hypothesis for us to obtain the exact generalization error curve,
  where the analytic functional calculus is essential.

  It is also of interest to ask whether a similar characterization holds for other non-analytic spectral algorithms.
  One particular algorithm is the spectral cut-off method (also known as truncated singular value decomposition)~\citep{bauer2007_RegularizationAlgorithms},
  whose filter function is not even continuous:
  \begin{align}
    \reg^{\mr{cut}}(z) =
    \begin{cases}
      z^{-1}, & z \geq \lambda, \\
      0, & z < \lambda.
    \end{cases},
    \qand
    \rem^{\mr{cut}}(z) = \mathbf{1}\left\{ z < \lambda \right\}.
  \end{align}
  Another similar example is the spectral clipping method, whose filter function is piecewise defined as
  \begin{equation}
    \reg^{\mr{clip}}(z) = \min(z^{-1}, \lambda^{-1})
    \qand
    \rem^{\mr{clip}}(z) = \max(0, 1 - z/\lambda).
  \end{equation}

  A possible future direction is to replace exact analyticity by a weaker condition that still permits comparable complex-analytic control on $D_\lambda$, or to develop a different argument beyond the current analytic functional calculus route.
  However, the difficulty here is that it is hard to approximate such filters by analytic ones while keeping the desired properties.
  We believe that new techniques are needed to handle this case, so we leave it as future work.
}

\appendix

\section{Auxiliary results}

\begin{table}[htbp]
  \centering
  \small
  \begin{tabular}{@{}p{0.3\linewidth}p{0.67\linewidth}@{}}
      \hline
      Notation & Meaning \\
      \hline
      $\caX,\caY,\rho,\mu$
      & input space, output space, data distribution on $\caX \times \caY$, and its marginal on $\caX$. \\
      $\caH,k,k_x$
      & RKHS, kernel, and representer $k_x = k(x,\cdot)$. \\
      $Z=\{(x_i,y_i)\}_{i=1}^n$, $X=(x_1,\dots,x_n)$
      & training sample and its input design. \\
      $T=S_kS_k^*$, $T_x$, $T_X=\frac{1}{n}\sum_{i=1}^n T_{x_i}$
      & population integral operator, rank-one sample operator, and empirical covariance operator. \\
      $(\lambda_j)_{j \geq 1}$, $(\mu_m,d_m,V_m)$
      & eigenvalues of $T$ counting multiplicities, and the distinct eigenvalues, multiplicities, and eigenspaces of $T$. \\
      $[\caH]^s$
      & interpolation/source space associated with $T^{s/2}$. \\
      $\lambda$
      & regularization parameter. \\
      $\reg(z),\ \rem(z)=1-z\reg(z)$
      & filter function and remainder function of a spectral algorithm. \\
      $\tau_{\max}$
      & qualification of the filter function. \\
      $\hat{g}_Z$, $\hat{f}_{\lambda}=\reg(T_X)\hat{g}_Z$
      & sample basis function and spectral estimator. \\
      $\tilde{f}_{\lambda}=\E(\hat{f}_{\lambda}\mid X)$
      & conditional mean of the estimator given the design. \\
      $\caN_{p,\varphi}(\lambda)$, $\caN_p(\lambda)$
      & generalized effective dimension and its KRR counterpart. \\
      $\caR_{\varphi}^2(\lambda;\fstar)$
      & deterministic bias term $\norm{\rem(T)\fstar}_{L^2}^2$. \\
      $\mathbf{Bias}^2(\lambda),\ \mathbf{Var}(\lambda)$
      & conditional bias and variance terms in the bias-variance decomposition. \\
      $\Gamma_\lambda$
      & $\lambda$-dependent contour used in the analytic functional argument. \\
      \hline
    \end{tabular}
  \caption{Main notation used throughout the paper.}
  \label{tab:notation}
\end{table}

This elementary proposition justifies the equivalence between \cref{eq:EDR} and \cref{eq:EDR_} in \cref{assu:EDR}.
\begin{proposition}
  \label{prop:DescendSequenceEquiv}
  Let $(a_j)_{j \geq 1}$ be a sequence of positive numbers decreasing to zero.
  Then,
  \begin{align*}
    a_j = \Theta(j^{-\beta}), \quad \Longleftrightarrow
    \quad \#\left\{ j : a_j \geq \lambda \right\} = \Theta(\lambda^{-1/\beta}) \qq{as} \lambda \to 0.
  \end{align*}
\end{proposition}
\begin{proof}

  We first note that $ \max \left\{ j : a_j \geq \lambda \right\} = \#\left\{ j : a_j \geq \lambda \right\}$
  since $(a_j)$ is decreasing.

  \noindent ($\Longrightarrow$):
  Suppose $c j^{-\beta} \leq a_j \leq C j^{-\beta}$.
  Then,
  \begin{align*}
    \max \left\{ j : a_j \geq \lambda \right\}
    \geq \max \left\{ j : c j^{-\beta} \geq \lambda \right\} = \Omega(\lambda^{-1/\beta}).
  \end{align*}
  On the other hand,
  \begin{align*}
    \max \left\{ j : a_j \geq \lambda \right\}
    &= \min \left\{ j : a_{j+1} < \lambda \right\} \\
    &\leq \min \left\{ j : C (j+1)^{-\beta} < \lambda \right\}
    = O(\lambda^{-1/\beta}).
  \end{align*}

  \noindent ($\Longleftarrow$):
  Let $N(\lambda) =  \max \left\{ j : a_j \geq \lambda \right\}$ and
  suppose $c \lambda^{-1/\beta} \leq N(\lambda) \leq C\lambda^{-1/\beta}$.
  We note that $N(\lambda) \geq j$ implies $a_j \geq \lambda$, so
  \begin{align*}
    a_j \geq \sup \left\{ \lambda : N(\lambda) \geq j \right\} \geq
    \sup \left\{ \lambda : c \lambda^{-1/\beta} \geq j \right\} = \Omega(j^{-\beta}).
  \end{align*}
  On the other hand, $N(\lambda) < j$ implies $a_j < \lambda$, so
  \begin{align*}
    a_j \leq \inf \left\{ \lambda : N(\lambda) < j \right\}
    \leq \inf \left\{ \lambda : C \lambda^{-1/\beta} < j \right\} = O(j^{-\beta}).
  \end{align*}
\end{proof}

\begin{proposition}
  \label{prop:EffectiveDimEstimation}
  Let $(\lambda_j)_{j \geq 1}$ be the descending sequence of eigenvalues counting multiplicities.
  Let us define
  \begin{align}
    \Phi(\ep) = \#\left\{ j : \lambda_j \geq \ep \right\}.
  \end{align}
  Suppose $\reg$ is a filter function satisfying \cref{eq:Filter_Reg} and \cref{eq:Filter_Rem}.
  Then, for any $p \geq 1$ and $\lambda > 0$, we have
  \begin{align}
    \label{eq:EffectiveDimEstimationRaw}
    2^{-p} \Phi( 2F_1 \lambda) \leq \caN_{p,\varphi}(\lambda) \leq p E^p \lambda^{-1} \int_0^{\lambda} \Phi(x) \dd x.
  \end{align}
  In particular, if \cref{assu:EDR} is satisfied, then
  \begin{align}
      \caN_{p,\varphi}(\lambda) = \Theta\left(\lambda^{-1/\beta}\right),\qq{as} \lambda \to 0.
  \end{align}
\end{proposition}
\begin{proof}
  We first deal with the upper bound.
  The property \cref{eq:Filter_Reg} of the filter function yields
  \begin{align*}
    z \reg(z) \leq E\frac{z}{\lambda + z} \leq E \min(1,\lambda^{-1} z)

  \end{align*}
  Consequently,
  \begin{align*}
      \caN_{p,\varphi}(\lambda) &=\sum_{j =1}^\infty \left[ \lambda_j \reg(\lambda_j) \right]^p
     \leq E^p \sum_{j =1}^\infty \min(1,\lambda^{-p} \lambda_j^p)
    = E^p \lambda^{-p} \sum_{j =1}^\infty \min(\lambda, \lambda_j)^p.
  \end{align*}
  Now, noticing that $p \int_0^a x^{p-1} \dd x = a^p$, we have
  \begin{align*}
      \sum_{j =1}^\infty \min(\lambda, \lambda_j)^p
    &= \sum_{j =1}^\infty  p\int_0^{\min(\lambda, \lambda_j)} x^{p-1} \dd x \\
    &= p\int_0^{\infty} \left( \sum_{j =1}^\infty \mathbf{1}\left\{ \min(\lambda, \lambda_j) \geq x \right\} \right) x^{p-1} \dd x \\
    &= p \int_0^{\lambda} \Phi(x) x^{p-1} \dd x.
  \end{align*}
  Therefore,
  \begin{align}
      \caN_{p,\varphi}(\lambda) \leq p E^p \lambda^{-p}\int_0^{\lambda} \Phi(x) x^{p-1} \dd x
    \leq p E^p \lambda^{-1} \int_0^{\lambda} \Phi(x) \dd x,
  \end{align}
  where the last inequality comes from $x/\lambda \leq 1$ when $x \leq \lambda$.

  For the lower bound, first, \cref{eq:Filter_Rem} gives $\rem(z) \leq F_1 \lambda z^{-1}$.
  Together with $z\reg(z) = 1 - \rem(z)$, we get
  \begin{align*}
    z\reg(z) = 1 - \rem(z) \geq 1 - F_1 \lambda z^{-1} \geq \frac{1}{2}, \quad \forall z \geq 2F_1 \lambda.
  \end{align*}
  Consequently, denoting $\tilde{\lambda} = 2F_1 \lambda$, we have
  \begin{align*}
      \caN_{p,\varphi}(\lambda) &=\sum_{j =1}^\infty \left[ \lambda_j \reg(\lambda_j) \right]^p \\
    & \geq \sum_{j \leq \Phi(\tilde{\lambda})} \left[ \lambda_j \reg(\lambda_j) \right]^p \\
    & \geq \sum_{j \leq \Phi(\tilde{\lambda})}2^{-p}
    = 2^{-p} \Phi( 2F_1 \lambda).
  \end{align*}

  Finally, if \cref{assu:EDR} is satisfied, then \cref{prop:DescendSequenceEquiv} implies
  $\Phi(\lambda) \asymp \lambda^{-1/\beta} $, where $1/\beta < 1$, so
  \begin{align*}
      \lambda^{-1} \int_0^{\lambda} \Phi(x) \dd x
    \leq C \lambda^{-1} \int_0^{\lambda} x^{-1/\beta} \dd x
    \leq C \lambda^{-1} \cdot \lambda^{1-1/\beta} = C \lambda^{-1/\beta},
  \end{align*}
  showing that the two sides in \cref{eq:EffectiveDimEstimationRaw} have the same order of $\lambda^{-1/\beta}$.
\end{proof}

\begin{proposition}
  \label{prop:SummationOrder}
  Let $(a_m)_{m\geq 1}$ be a descending sequence of positive numbers and
  $(b_m)_{m\geq 1}, (c_m)_{m\geq 1}$ be two sequences of positive numbers satisfying
  \begin{align*}
    \sum_{k=1}^{m} b_k \leq  \sum_{k=1}^{m} c_k,\quad \forall m \geq 1.
  \end{align*}
  Then, for any $N \geq 1$,
  \begin{align*}
    \sum_{m=1}^N a_m b_m \leq \sum_{m=1}^N a_m c_m.
  \end{align*}
\end{proposition}
\begin{proof}
  Using Abel's summation formula, we have
  \begin{align*}
    \sum_{m=1}^N a_m b_m &= \sum_{m=1}^{N-1} (a_m - a_{m+1}) \sum_{k=1}^{m} b_k + a_N  \sum_{k=1}^{N} b_k \\
    &\leq \sum_{m=1}^{N-1} (a_m - a_{m+1}) \sum_{k=1}^{m} c_k + a_N  \sum_{k=1}^{N} c_k  = \sum_{m=1}^N a_m c_m.
  \end{align*}
\end{proof}

\subsection{General filter functions}
The following is a well-known elementary property related to $\reg^{\mr{KR}}$:
\begin{proposition}
  For $\lambda >0$ and $\alpha \in [0,1]$, we have
  \begin{align}
    \label{eq:FilterKRRProp}
    \frac{z^\alpha}{z + \lambda} \leq \lambda^{\alpha-1}.
  \end{align}
\end{proposition}

\begin{lemma}
  \label{lem:Filter_MoreControl}
  Let $\reg$ be a filter function defined in \cref{def:filter}.
  Then, for $s \in [0,1]$,
  \begin{align}
    \label{eq:Filter_Regs}
    \sup_{z \in [0,\kappa^2]}\reg(z) z^s \leq E \lambda^{s-1}.
  \end{align}
  Also, suppose  \cref{eq:Filter_Rem} is satisfied for $\tau$, then the constant $F_r$ satisfies
  \begin{align}
    F_{r} \leq F_0^{1-\frac{r}{\tau}} F_{\tau}^{\frac{r}{\tau}},\quad \forall r \in [0,\tau].
  \end{align}
  Moreover, when $z \leq \frac{\lambda}{2E}$, we have $\rem(z) \geq 1/2$.
\end{lemma}
\begin{proof}
  The first inequality is a consequence of \cref{eq:Filter_Reg} and \cref{eq:FilterKRRProp}.
  The second one comes from
  \begin{align*}
    z^{r} \rem(z) = \rem(z)^{1-\frac{r}{\tau}} \left( z^{\tau}\rem(z) \right)^{\frac{r}{\tau}}
    \leq F_0^{1-\frac{r}{\tau}} (F_{\tau}\lambda^{\tau})^{\frac{r}{\tau}}
    = F_0^{1-\frac{r}{\tau}} F_{\tau}^{\frac{r}{\tau}} \lambda^r.
  \end{align*}
  For the last claim, note that $\reg(z) \leq E\lambda^{-1}$, so when $z \leq \frac{\lambda}{2E}$,
  \begin{align*}
    \rem(z) = 1-z \reg(z) \geq 1 - \frac{\lambda}{2E} E\lambda^{-1} = \frac{1}{2}.
  \end{align*}
\end{proof}

\subsection{Concentration inequalities}
See, for example, \citet[Proposition 2.5]{wainwright2019_HighdimensionalStatistics}
for the standard Hoeffding inequality.
\begin{lemma}[Hoeffding's inequality]
  \label{lem:HoeffdingInequality}
  Let $\xi,\xi_1,\dots,\xi_n$ be i.i.d.\ random variables such that $\abs{\xi} \leq B$ almost surely.
  Then, for any $\delta \in (0,1)$, with probability at least $1-\delta$ we have
  \begin{align}
    \abs{\frac{1}{n} \sum_{i=1}^n \xi_i - \E \xi} \leq \sqrt {\frac{2B^2}{n} \ln \frac{2}{\delta}}.
  \end{align}
\end{lemma}

The following inequality about vector-valued random variables is well-known in the literature~\citep{caponnetto2007_OptimalRates}.
\begin{lemma}
  \label{lem:ConcenHilbert}
  Let $H$ be a separable Hilbert space.
  Let $\xi,\xi_1,\dots,\xi_n$ be i.i.d.\ random variables taking values in $H$.
  Assume that
  \begin{align}
    \label{eq:HilbertConcenCondition}
    \E \norm{\xi - \E \xi}_{H}^m \leq \frac{1}{2}m!\sigma^2 L^{m-2},\quad \forall m = 2,3,\dots.
  \end{align}
  Then, for any fixed $\delta \in (0,1)$, one has
  \begin{align}
    \bbP \left\{ \norm{\frac{1}{n}\sum_{i=1}^n \xi_i - \E \xi}_{H}
    \leq 2\left( \frac{L}{n} + \frac{\sigma}{\sqrt{n}} \right) \ln \frac{2}{\delta} \right\}
    \geq 1- \delta.
  \end{align}
  In particular, a sufficient condition for \cref{eq:HilbertConcenCondition} is
  \begin{align*}
    \norm{\xi}_{H} \leq \frac{L}{2}\ \text{a.s.},
    \qquad
    \E \norm{\xi}_H^2 \leq \sigma^2.
  \end{align*}
\end{lemma}

The following Bernstein inequality for random self-adjoint Hilbert-Schmidt operators is commonly used in the literature~(e.g., \citet[Lemma B.5]{li2023_KernelInterpolation}).
It is a slightly modified version of its original form~\citep[Theorem 7.7.1]{tropp2015_IntroductionMatrix}.

\begin{lemma}
  \label{lem:ConcenBernstein}
  Let $H$ be a separable Hilbert space.
  Let $A_1,\dots,A_n$ be i.i.d.\ random variables taking values in the space of self-adjoint Hilbert-Schmidt operators on $H$
  such that $E (A_1) = 0$, $\norm{A_1} \leq L$ almost surely for some $L > 0$ and
  $E (A_1^2) \preceq V$ for some positive trace-class operator $V$.
  Then, for any $\delta \in (0,1)$, with probability at least $1-\delta$ we have
  \begin{align*}
    \norm{\frac{1}{n}\sum_{i=1}^n A_i} \leq \frac{2LB}{3n} + \left(\frac{2\norm{V}B}{n}\right)^{1/2},
    \qq{where} B = \ln \frac{4 \Tr V}{\delta \norm{V}}.
  \end{align*}
\end{lemma}

\section{Omitted proofs}

\subsection{Regular RKHS}

\begin{proof}[Proof of \cref{prop:EmbeddingIdx}]
  It is shown in \citet[Theorem 9]{fischer2020_SobolevNorm} that the norm of the embedding
  \begin{align}
    \label{eq:EMB_And_InfNorm}
    \norm{[\mathcal{H}]^\alpha \hookrightarrow L^{\infty}(\mathcal{X},\mu)}
    &= \norm{k^\alpha_{\mu}}_{L^\infty} \\
    &\coloneqq \operatorname*{ess~sup}_{x \in \caX,~\mu} \sum_{i =1}^\infty \lambda_i^{\alpha} e_i(x)^2 \\
    &= \operatorname*{ess~sup}_{x \in \caX,~\mu} \sum_{m=1}^\infty \mu_m^\alpha \sum_{l=1}^{d_m} \abs{e_{m,l}(x)}^2.
  \end{align}
  Then, recalling \cref{assu:RegularRKHS}, \cref{prop:SummationOrder} with a limit argument yields
  \begin{align*}
    &\quad  \sum_{m=1}^\infty \mu_m^\alpha \sum_{l=1}^{d_m} \abs{e_{m,l}(x)}^2
    \leq M \sum_{m=1}^\infty d_m \mu_m^\alpha \\
    &= M \sum_{i=1}^\infty \lambda_i^\alpha
    \leq M \sum_{i=1}^\infty i^{-\alpha \beta},
  \end{align*}
  so the norm of the embedding is finite as long as $\alpha > 1/\beta$.
\end{proof}

\subsubsection{Dot-product kernel on the sphere}
\label{subsubsec:DotProductSphere}

Let $\caX = \bbS^d$ be the $d$-dimensional sphere and $\mu$ be the uniform measure on $\bbS^d$.
  Then, classical results~\citep{dai2013_ApproximationTheory} show that
the eigen-decomposition of the spherical Laplacian $\Delta_{\bbS^d}$ gives an orthogonal direct sum decomposition
\begin{align*}
  L^2(\bbS^d) = \bigoplus_{m=0}^{\infty} \caH^d_m(\bbS^d),
\end{align*}
where $\caH^d_m(\bbS^d)$ consists of the restrictions of degree-$m$ homogeneous harmonic polynomials with $d+1$ variables to $\bbS^d$
and $\caH^d_m(\bbS^d)$ is an eigenspace of $\Delta_{\bbS^d}$ associated with eigenvalue $-m(m+d-1)$.
Moreover, the dimension of $\caH^d_m(\bbS^d)$ is given by
\begin{align*}
  a_m \coloneqq \dim \caH^d_m(\bbS^d) = \binom{m+d}{m} - \binom{m-2+d}{m-2} \asymp m^{d-1},
\end{align*}
and $\sum_{k \leq m} a_k = \binom{m+d}{m} + \binom{m-1+d}{m-1} \asymp m^d$.

Moreover, the reproducing kernel $Z_m(x,y)$ of $\caH^d_m(\bbS^d)$ is well-defined and unique,
which can be given explicitly by
\begin{align*}
  Z_m(x,y) = \sum_{l=1}^{a_m} Y_{m,l}(x)Y_{m,l}(y),
\end{align*}
where $\{Y_{m,l}\}_{l=1}^{a_m}$ is an arbitrary orthonormal basis of $\caH^d_m(\bbS^d)$.
Let us denote by $C_m^\lambda$ the Gegenbauer polynomial, which is often defined by the following power series
\begin{align}
  \sum_{m=0}^{\infty} C_m^{\lambda}(t) \alpha^m = \frac{1}{(1-2t\alpha+\alpha^2)^\lambda}.
\end{align}
Then, when $d \geq 2$, we have
\begin{align}
  \label{eq:Zonal_Gegenbauer}
  Z_m(x,y) = \frac{m+\lambda}{\lambda} C_m^{\lambda}(u),\quad u = \ang{x,y},\; \lambda = \frac{d-1}{2}.
\end{align}
Also, \citet[Corollary 1.27]{dai2013_ApproximationTheory} shows that
\begin{align}
  \label{eq:Zonal_Maximal}
  Z_m(x,x) = \frac{m+\lambda}{\lambda} C_m^{\lambda}(1) = \dim \caH^d_m(\bbS^d).
\end{align}
Furthermore, we have the following Funk-Hecke formula~\citep[Theorem 1.2.9]{dai2013_ApproximationTheory}.
\begin{proposition}[Funk-Hecke formula]
  Let $d \geq 3$ and $h$ be an integrable function such that $\int_{-1}^1 \abs{h(t)} (1-t^2)^{(d-2)/2} \dd t$ is finite.
  Then for every $Y_m \in \caH^d_m(\bbS^d)$,
  \begin{align}
    \label{eq:C_FunkHecke}
    \int_{\bbS^d} h(\ang{x,y}) Y_m(y) \dd \mu(y) = \mu_m(h) Y_m(x),\quad \forall x \in \bbS^{d},
  \end{align}
  where $\mu_m(h)$ is a constant defined by
  \begin{align*}
    \mu_m(h) = \omega_d \int_{-1}^1 h(t) \frac{C_m^\lambda(t)}{C_m^\lambda(1)} (1-t^2)^{\frac{d-2}{2}} \dd t,
  \end{align*}
  and $\omega_d$ is the surface area of the unit sphere $\bbS^d$.
\end{proposition}

Comparing \cref{eq:IntegralOperator} and \cref{eq:C_FunkHecke}, we conclude that $\caH^d_m(\bbS^d)$
is an eigenspace of $T$ corresponding to the eigenvalue $\mu_m = \mu_m(h)$.
Consequently, we get $Z_m = k_m$ and hence according to \cref{eq:Zonal_Maximal}, we have
\begin{align*}
  k_m(x,x) = Z_m(x,x) = \dim \caH^d_m(\bbS^d) = \dim V_m,
\end{align*}
so \cref{assu:RegularRKHS} holds for $M=1$.

\subsubsection{Dot-product kernel on the ball}
\label{subsubsec:DotProductBall}

The case of the ball is similar to the case of the sphere.
We refer to \citet[Section 11]{dai2013_ApproximationTheory} for more details.
Let us consider the $d$-dimensional unit ball $\caX = \bbB^d = \{x \in \R^{d+1} : \norm{x} \leq 1\}$
and let $\mu$ be proportional to the classical weight $W(x) = (1-\norm{x}^2)^{-1/2}$.
Let us denote by $V_m^d$ the space of orthogonal polynomials of degree exactly $m$
with respect to the inner product
\begin{align*}
  \ang{f,g}_{W} = c_{W} \int_{\bbB^d} f(x) g(x) W(x) \dd x,
\end{align*}
where $c_W$ is a normalization constant.
Then, we have
\begin{align*}
  \dim V_m^d = \binom{m+d-1}{m}.
\end{align*}
Moreover, we have the following analog of the Funk–Hecke formula~\citep[Theorem 11.1.9]{dai2013_ApproximationTheory}:
\begin{proposition}[Funk-Hecke formula]
  Let $\lambda = \frac{d-1}{2}$ and $h$ be an integrable function such that $\int_{-1}^1 \abs{h(t)} (1-t^2)^{\lambda-1/2} \dd t$ is finite.
  Then for every $P_m \in V^d_m$,
  \begin{align}
    c_W
    \int_{\bbB^d} h(\ang{x,y}) P_m(y) W(y) \dd y = \mu_m(h) P_m(x),\quad \forall x \in \bbB^{d},
  \end{align}
  where $\mu_m(h)$ is a constant defined by
  \begin{align*}
    \mu_m(h) = c_{\lambda} \int_{-1}^1 h(t) \frac{C_m^\lambda(t)}{C_m^\lambda(1)} (1-t^2)^{\lambda-\frac{1}{2}} \dd t,
  \end{align*}
  and $c_{\lambda}$ is a constant chosen so that $\mu_0(1) = 1$.
\end{proposition}

Consequently, for any inner product kernel $k(x,y) = h(\ang{x,y})$,
$V^d_m$ is an eigenspace of $T$ corresponding to the eigenvalue $\mu_m = \mu_m(h)$.

Moreover, \citet[Corollary 11.1.8]{dai2013_ApproximationTheory} shows that the reproducing kernel $k_m$ for $V^d_m$ is
\begin{align*}
  k_m(x,y) = \frac{m+\lambda}{2\lambda} \left[ C_m^{\lambda}(\ang{x,y}+x_{d+1} y_{d+1}) + C_m^{\lambda}(\ang{x,y} - x_{d+1} y_{d+1}) \right],
\end{align*}
where $x_{d+1} = \sqrt {1-\norm{x}^2}, y_{d+1} = \sqrt {1-\norm{y}^2}$.
Consequently,
\begin{align*}
  k_m(x,x) = \frac{m+\lambda}{2\lambda} \left( C_m^{\lambda}(1) + C_m^{\lambda}(2\norm{x}^2 - 1) \right)
  \leq \frac{m+\lambda}{\lambda} C_m^{\lambda}(1) \\
  = \frac{m+\lambda}{\lambda} \frac{\Gamma(2\lambda + m)}{\Gamma(2\lambda) m!}
\end{align*}
  where we use \citet[Eq. (B.2.2)]{dai2013_ApproximationTheory} for $C_m^{\lambda}(1)$.
Expanding the expression with $\lambda = (d-1)/2$, we get
\begin{align*}
  \frac{m+\lambda}{\lambda} \frac{\Gamma(2\lambda + m)}{\Gamma(2\lambda) m!} / \dim V^d_m
  = \frac{2m+d-1}{m+d-1} \leq 2,
\end{align*}
so \cref{assu:RegularRKHS} also holds with $M = 2$.

\subsection{Analytic filter functions}
\label{subsec:FilterFunction}

To further analyze the properties of filter functions in the complex plane, let us first recall some results in complex analysis.

\begin{proposition}[Maximum modulus principle]
  \label{prop:max_modulus}
  Let $f$ be an analytic function on an open set $\Omega$ and $K \subset \Omega$ be a compact set.
  Then
  \begin{align*}
    \sup_{z \in K} \abs{f(z)} = \sup_{z \in \partial K} \abs{f(z)}.
  \end{align*}
\end{proposition}

\begin{proposition}
  The filter functions of KRR (\cref{example:KRR}), iterated ridge (\cref{example:IteratedRidge}) and gradient flow (\cref{example:GradientFlow}) satisfy \cref{assu:Filter}.
\end{proposition}
\begin{proof}
  In the case of KRR, this conclusion is trivial.\\
  ($I$) Condition (C1) for iterated ridge and gradient flow: Define $H=\{z\in\bbC:\Re(z)\geq-\frac{\lambda}{2}\}$.
  Note that the filter functions of gradient flow and iterated ridge are both of the form
  $\reg(z) = g(u)/(\lambda u)$ for some analytic function $g(u)$ on $H$, where $u=z/\lambda$. Specifically, for iterated ridge, $g(u)=1-\frac{1}{(u+1)^p}$. For gradient flow, $g(u)=1-e^{-u}$. Note that $\frac{g(u)}{u}$ can be extended to an analytic function on $H$.
  Hence,
  \begin{align*}
    \abs{(z+\lambda)\reg(z)} = \abs{(u+1) g(u)/u}\eqqcolon G(u)>0
  \end{align*}
  can also be viewed as an analytic function on $H$.
  In the case of iterated ridge, $\limsup_{\abs{u}\rightarrow\infty}G(u)=\limsup_{\abs{u}\rightarrow\infty}g(u)=1$ is finite. Therefore, $\sup_{u\in H}G(u)<\infty$, so condition (C1) holds. The filter function of gradient flow satisfies the same condition since $\limsup_{\abs{u}\rightarrow\infty}g(u)\leq \sup_{u\in H}g(u)\leq 1+\sqrt{e}<\infty$.\\
  ($II$) Condition (C2) for iterated ridge and gradient flow: In the case of iterated ridge,
  \begin{equation*}
    \abs{\frac{z+\lambda}{\lambda}\rem^{\mr{IT},p}(z)}=\left(\frac{\abs{\lambda}}{\abs{z+\lambda}}\right)^{p-1}\leq 2^{p-1}
  \end{equation*}
  for all $z$ satisfying $\Re(z)\geq-\frac{\lambda}{2}$.
  For gradient flow, when $\Re(z)\leq 0$ and $z\in D_\lambda$,
  \begin{equation*}
    \abs{\frac{z+\lambda}{\lambda}\rem^{\mathrm{GF}}(z)}=\abs{\frac{z+\lambda}{\lambda}}e^{\Re(-\frac{z}{\lambda})}\leq\frac{5}{2}\sqrt{e}.
  \end{equation*}
  If $\Re(z)>0$ and $z\in D_\lambda$,
  \begin{equation*}
    \abs{\frac{z+\lambda}{\lambda}\rem^{\mathrm{GF}}(z)}\leq2(1+\frac{\Re(z)}{\lambda})e^{\Re(-\frac{z}{\lambda})}.
  \end{equation*}
  Note that the function $f(x)\coloneqq (1+x)e^{-x}$ is monotonically decreasing on $[0,+\infty)$. Hence, $\abs{\frac{z+\lambda}{\lambda}\rem^{\mathrm{GF}}(z)}\leq 2$.
\end{proof}

\begin{proposition}
  The filter function of gradient descent (\cref{example:GradientDescent}) satisfies \cref{assu:Filter}.
\end{proposition}
\begin{proof}
  For condition (C1), note that $(z+\lambda)\reg^{\mathrm{GD}}(z)$ can be extended to an analytic function on $D_\lambda$.
  By \cref{prop:max_modulus}, it suffices to prove that $\sup_{z\in\Gamma_\lambda}(z+\lambda)\reg^{\mathrm{GD}}(z)$ is controlled by a constant independent of $\lambda$. Indeed, for all $z\in\Gamma_\lambda$, $\frac{\abs{z+\lambda}}{\abs{z}}\leq 1+\frac{\abs{\lambda}}{\abs{z}}\leq 3$ and $\abs{1-\eta z}^\frac{1}{\eta\lambda}\leq\abs{1+\hf\eta\lambda}^\frac{1}{\eta\lambda}\leq \sqrt{e}$. Hence, condition (C1) also holds for $\reg^{\mathrm{GD}}$.

  For condition (C2), we need to show that
  \begin{equation*}
    \sup_{z\in D_\lambda,\lambda\in(0,\varepsilon)}\abs{\frac{z+\lambda}{\lambda}\rem^{\mr{GD}}(z)}=\sup_{z\in D_\lambda,\lambda\in(0,\varepsilon)}\abs{\frac{z+\lambda}{\lambda}(1-\eta z)^\frac{1}{\eta\lambda}}<\infty
  \end{equation*}
  for some constant $\varepsilon>0$.
  In fact, when $\abs{1-\eta z}\geq 1$,
  \begin{equation*}
    \abs{\frac{z+\lambda}{\lambda}(1-\eta z)^\frac{1}{\eta\lambda}}\leq \frac{3}{2}\abs{1+\hf\eta\lambda}^\frac{1}{\eta\lambda}\leq \frac{3}{2}\sqrt{e}
  \end{equation*}
  for all $z\in D_\lambda$. When $\abs{1-\eta z}< 1$,
  \begin{equation*}
    \abs{\frac{z+\lambda}{\lambda}(1-\eta z)^\frac{1}{\eta\lambda}}\leq 1+\abs{\frac{z}{\lambda}(1-\eta z)^\frac{1}{\eta\lambda}}\leq 1+\abs{z(1-\eta z)^\frac{1}{\eta}}\leq \frac{3}{2}+2\kappa^2
  \end{equation*}
  for all $z\in D_\lambda$ and $\lambda\in(0,1)$.
\end{proof}

\subsection{Source condition on the regression function}
\label{subsec:SourceCondition}

\begin{proof}[Proof of \cref{example:source1}]
  With \cref{assu:EDR} and $\sum_{k=1}^m d_k\asymp m^\gamma$, we have $\mu_m\asymp m^{-\gamma\beta}$. Select $s=\frac{p}{\beta}$. Then, for all $t<s$,
  \begin{equation*}
    \sum_{m=1}^\infty \frac{\bar{f}_m^2}{\mu_m^t}=O\left(\sum_{m=1}^\infty m^{-(\gamma p+1)}m^{t\gamma\beta}\right)=O\left(\sum_{m=1}^\infty m^{-1-\gamma(p-\beta t)}\right)<\infty.
  \end{equation*}
  Hence, $\fstar\in[\caH]^t$ for all $t<s$.

  On the other hand,
  \begin{align*}
    \sum_{m : \mu_m\leq\lambda}\bar{f}_m^2
    \geq c \sum_{m \geq C \lambda^{-(\gamma\beta)^{-1}}} m^{-\gamma p-1}
    = \Omega(\lambda^{\frac{p}{\beta}})=\Omega(\lambda^s).
  \end{align*}

  Finally, if $f_j\asymp j^{-(p+1)/2}$, we have
  \begin{equation*}
    \bar{f}_m^2=\sum_{j=D_{m-1}}^{D_m}f_j^2\asymp \int_{(m-1)^\gamma}^{m^\gamma}x^{-p-1}dx\asymp m^{-(\gamma p+1)}
  \end{equation*}
  where $D_m\coloneqq\sum_{k=1}^md_k\asymp m^\gamma$.
\end{proof}

\begin{proof}[Proof of \cref{example:source2}]
  Let $s=\frac{p}{q\beta}$.
  Then, with \cref{assu:EDR}, for all $t<s$,
  \begin{equation*}
    \sum_{j=1}^\infty \lambda_j^{-t} \abs{f_j}^2  =
    \sum_{l=1}^\infty \lambda_{j(l)}^{-t} \abs{f_{j(l)}}^2
    \leq C \sum_{l=1}^\infty l^{t q\beta} l^{-(p+1)}
    \leq C \sum_{l=1}^\infty l^{-1 + (p - t q\beta)},
  \end{equation*}
  where we recall that $j(l) \asymp l^{q}$ and $\abs{f_{j(l)}} \asymp l^{-(p+1)/2}$,
  so $\fstar\in[\caH]^t$ for all $t<s$.

  On the other hand,
  \begin{align*}
    \sum_{j : \lambda_j \leq\lambda} \abs{f_j}^2
    = \sum_{l : \lambda_{j(l)} \leq \lambda} \abs{f_{j(l)}}^2
    \geq c \sum_{l \geq C \lambda^{-1/(q\beta)}} l^{-(p+1)}
    = \Omega(\lambda^{\frac{p}{q\beta}}) = \Omega(\lambda^s).
  \end{align*}
\end{proof}

\subsection{Proof of \cref{lem:RegApprox}}

The proof follows the same idea as the proofs of Lemmas A.5 and A.10 in \citet{li2023_AsymptoticLearning},
but we establish the results for spectral algorithms and improve the estimates in \cref{lem:NormsN1N2} using the regular RKHS property.
We first rewrite the quantity in \cref{eq:RegApprox} as a centered empirical process after preconditioning by $T_\lambda^{-\hf}$, and then control this process by concentration.
The proof will be divided into two cases: $t > 1/\beta$ and $t \leq 1/\beta$.
For the first case, the embedding result yields enough $L^\infty$ control on $\fstar - \flam$, so that we can directly apply the Bernstein inequality to the transformed empirical process.
For the second case, such a uniform bound is no longer available, so we truncate the large values of $\fstar$, apply the same concentration argument to the truncated process, and then show that the truncation remainder is negligible.

\paragraph{Proof for the case $t > 1/\beta$}
By the inclusion relation of $[\caH]^t$, we can replace $t$ with $\tilde{t} = \min(t,2\tau)$ so that $t \leq 2\tau$.
We will use the Bernstein inequality in \cref{lem:ConcenHilbert}.
Let us denote
\begin{align}
  \label{eq:a_Proof_xi}
  \xi(x) =  T_\lambda^{-\hf}(K_{x} \fstar(x) - T_{x} \flam).
\end{align}
Then, we have
\begin{align*}
  T_\lambda^{-\hf} \left[\left( \gtl - T_X f_\lambda^*\right)-\left(g-T f_\lambda^*\right)\right]
  = \frac{1}{n}\sum_{i=1}^n \xi(x_i) - \E_{x \sim \mu} \xi(x).
\end{align*}
Moreover, we have
\begin{align}
  \label{eq:a_ProofApprox1_Moment}
  \E \norm{\xi(x)}_{\mathcal{H}}^{m} &= \E \norm{ T_\lambda^{-\hf} K_{x}(\fstar(x) - \flam(x)) }_{\caH}^{m} \notag \\
  &\leq \E \left[ \norm{T_\lambda^{-\hf} k_x}_{\mathcal{H}}^{m} \cdot  \E \big( \abs{\fstar(x) - \flam(x)}^m ~\big|~ x\big) \right].
\end{align}
The first term in \cref{eq:a_ProofApprox1_Moment} is bounded through \cref{eq:RegPhiKx} and \cref{lem:EffectiveDimEstimationPowerlaw}:
\begin{align*}
  \norm{T_\lambda^{-\hf} k_x}_{\mathcal{H}}^2
  \leq M \caN_{1}(\lambda) \leq C \lambda^{-1/\beta}.
\end{align*}
For the second term, since $t > 1/\beta$, using the embedding property \cref{prop:EmbeddingIdx} together with \cref{lem:FApproximation},
for $\alpha \in (1/\beta, t) $, we have
\begin{align*}
  \norm{\flam - \fstar}_{L^\infty} \leq M_{\alpha} \norm{\flam -\fstar}_{[\caH]^{\alpha}}
  \leq M_{\alpha}F_{(t-\alpha)/2} \norm{\fstar}_{[\caH]^t} \lambda^{\frac{t-\alpha}{2}}
  \leq C M_{\alpha} R \lambda^{\frac{t-\alpha}{2}}
\end{align*}
Moreover, \cref{lem:FApproximation} also implies
\begin{align*}
  \E \abs{\flam(x) - \fstar(x)}^2 = \norm{\flam(x) - \fstar(x)}_{L^2}^2 \leq
  F_{t/2}^2 \norm{\fstar}_{[\caH]^t}^2 \lambda^{t} \leq C R^{2} \lambda^{t}.
\end{align*}
Plugging in these estimations in \cref{eq:a_ProofApprox1_Moment}, we get
\begin{align*}
  \label{eq:a_ProofApprox1_Bernstein}
  \text{\cref{eq:a_ProofApprox1_Moment}}
  &\leq (C \lambda^{-1/(2\beta)})^{m} \cdot  \norm{\flam - \fstar}_{L^{\infty}}^{m-2} \cdot
  \E \abs{\flam(x) - \fstar(x)}^2 \notag \\
  & \leq (C \lambda^{-1/(2\beta)})^{m} \cdot \left( C M_{\alpha} R \lambda^{\frac{t-\alpha}{2}} \right)^{m-2}
  \cdot C R^2 \lambda^{t} \\
  & \leq \frac{1}{2} m! \left( C R^2 \lambda^{t-1/\beta} \right) \left( C M_{\alpha} R \lambda^{\frac{t-\alpha-1/\beta}{2}}  \right)^{m-2}

\end{align*}
Consequently, applying \cref{lem:ConcenHilbert} with
\begin{align*}
  \sigma^2 = C R^2 \lambda^{t-1/\beta},\qquad L = C M_{\alpha} R \lambda^{\frac{t-\alpha-1/\beta}{2}}.
\end{align*}
yields
\begin{align*}
  \norm{T_\lambda^{-\hf} \left[\left(\gtl-T_X \flam \right)-\left(g^* - T \flam \right)\right]}_\caH
  \leq C R \sqrt {\frac{\lambda^{-1/\beta}}{n}} \left( 1 + M_\alpha \sqrt {\frac{\lambda^{-\alpha}}{n}}\right) \lambda^{t/2}.
\end{align*}
Since $\lambda = \Omega(n^{-\theta})$ for some $\theta < \beta$, choosing $\alpha > 1/\beta$ sufficiently close to $1/\beta$ yields
the desired result.

\paragraph{Proof for the case $t \leq 1/\beta$}
For this case, we apply a truncation technique.
To bound the extra error terms caused by truncation, we use the following proposition about the $L^q$ embedding of the RKHS
~\citep[Theorem 5]{zhang2023_OptimalityMisspecified}.
\begin{proposition}
  \label{prop:LqEMB}
  Under \cref{assu:RegularRKHS}, for any $0 < s \leq \alpha_0$ and $\alpha > \alpha_0$, we have embedding
  \begin{align}
  [\caH]
    ^s \hookrightarrow L^{q}(\caX,\dd \mu),\quad q = \frac{2\alpha}{\alpha - s}.
  \end{align}
\end{proposition}

Now, let us consider $\Omega_{B} = \left\{ x \in \caX : \abs{\fstar(x)} \leq BR \right\}$ and
\begin{align*}
  \bar{\xi}(x) = \xi(x) \bm{1}_{\Omega_B}(x)
\end{align*}
as $\xi$ is defined in \cref{eq:a_Proof_xi}, where the choice of $B$ will be determined later.
Then,
\begin{align}
  \label{eq:a_Proof_TruncDecomp}
  \begin{aligned}
    \norm{\frac{1}{n} \sum_{i=1}^n \xi(x_i)-\mathbb{E} \xi(x)}_{\caH}
    & \leq \norm{\frac{1}{n} \sum_{i=1}^n \bar{\xi}(x_i) -  \E \bar{\xi}(x)}_{\caH}
    + \norm{\frac{1}{n} \sum_{i=1}^n \xi(x_i) \bm{1}_{\Omega_B^\complement}(x_i)}_{\caH} \\
    & \quad + \norm{\E \xi(x)\bm{1}_{\Omega_B^\complement}(x)}_{\caH}.
  \end{aligned}
\end{align}

For the first term in \cref{eq:a_Proof_TruncDecomp}, we can repeat the same argument in the first case with the extra bound
\begin{align*}
  \norm{\bm{1}\left\{ x \in \Omega_B \right\}(\flam - \fstar)}_{L^\infty} &\leq
  \norm{\flam }_{L^\infty} + \norm{\bm{1}\left\{ x \in \Omega_B \right\} \fstar}_{L^\infty} \\
  & \leq M_{\alpha} \norm{\flam }_{[\caH]^\alpha} + BR \\
  & \leq M_{\alpha} R \lambda^{\frac{t-\alpha}{2}} + BR,
\end{align*}
where we apply \cref{lem:FApproximation} in the last inequality.
Consequently, we get
\begin{align}
  \label{eq:a_Proof_TruncDecomp_A}
  \norm{\frac{1}{n} \sum_{i=1}^n \bar{\xi}(x_i) -  \E \bar{\xi}(x)}_{\caH}
  \leq C R \sqrt {\frac{\lambda^{-1/\beta}}{n}} \left( 1 +  \frac{M_\alpha\lambda^{-\alpha/2} + B \lambda^{-t/2} }{\sqrt {n}}\right) \lambda^{t/2},
\end{align}

For the second term in \cref{eq:a_Proof_TruncDecomp}, \cref{prop:LqEMB} together with Markov's inequality yields
\begin{align}
  \label{eq:a_Proof_TruncDecomp_Prob}
  \bbP_{\mu}(\Omega_B^\complement) = \bbP_{x \sim \mu}\left\{ \abs{\fstar(x)} > BR \right\}
  \leq (BR)^{-q} \norm{\fstar}_{L^q}^q,
\end{align}
where $q = \frac{2\alpha}{\alpha - s}$.
Then,
\begin{align*}
  \bbP\left\{ x_i \in \Omega_B,~i=1,\dots,n \right\}
  = \left( 1 - \bbP_{\mu}(\Omega_B^\complement) \right)^n
  \geq (1 - (BR)^{-q} \norm{\fstar}_{L^q}^q )^n.
\end{align*}
Consequently, as long as
\begin{align}
  \label{eq:a_Proof_TruncDecomp_B}
  B^{-q} = o\left( \frac{1}{n} \right), \qq{or equivalently,} B = \omega\left( n^{1/q} \right),
\end{align}
we have $\bbP\left\{ x_i \in \Omega_B,~i=1,\dots,n \right\} \to 1$ and thus
the second term in \cref{eq:a_Proof_TruncDecomp} vanishes with high probability when $n$ is large enough.

For the third term in \cref{eq:a_Proof_TruncDecomp},
\begin{align}
  \notag
  \norm{\E \xi(x)\bm{1}_{\Omega_B^\complement}(x)}_{\caH}
  & \leq \E \norm{\xi(x)\bm{1}_{\Omega_B^\complement}(x) }_{\caH} \\
  \notag
  & = \E \left[ \bm{1}_{\Omega_B^\complement}(x) (\fstar(x) - \flam(x)) \norm{\reg^{1/2}(T)k_x}_{\caH} \right] \\
  \notag
  & \stackrel{(a)}{\leq} C \lambda^{-1/(2\beta)} \E \left[ \bm{1}_{\Omega_B^\complement}(x)(\fstar(x) - \flam(x)) \right] \\
  \notag
  & \stackrel{(b)}{\leq} C \lambda^{-1/(2\beta)} \norm{\fstar-\flam}_{L^2} \bbP(\Omega_B^\complement)^{\frac{1}{2}} \\
  \label{eq:a_Proof_TruncDecomp_C}
  & \stackrel{(c)}{\leq} C \lambda^{-1/(2\beta)} \lambda^{t/2} B^{-q/2} \norm{\fstar}_{L^q}^{q/2}
  = C \norm{\fstar}_{L^q}^{q/2}\lambda^{t/2} \sqrt {\lambda^{-1/\beta} B^{-q}},
\end{align}
where in $(a)$ we apply \cref{eq:RegPhiKx}, in $(b)$ we apply Cauchy-Schwarz inequality
and in $(c)$ we apply \cref{lem:FApproximation} for the $L^2$ norm and \cref{eq:a_Proof_TruncDecomp_Prob} for the probability.

Finally, since $\theta < \beta$, we can choose $\alpha \in (\beta^{-1},\theta^{-1})$.
Then we have
\begin{align*}
  1-\theta/\beta > 1-\alpha \theta > 0,
\end{align*}
and we can choose some $l$ such that
\begin{align*}
  \frac{1}{q} = \frac{1}{2} (1 - \frac{t}{\alpha}) < l < \frac{1-\theta t}{2}.
\end{align*}
Now we set $B = n^l$.
Then \cref{eq:a_Proof_TruncDecomp_B} immediately holds.
For the term \cref{eq:a_Proof_TruncDecomp_A}, we have
\begin{align*}
  \sqrt {\frac{\lambda^{-1/\beta}}{n}}
  &= O\left( n^{-(1-\theta/\beta)/2} \right), \\
  \frac{\lambda^{-\alpha/2}}{\sqrt{n}}
  &= O\left( n^{-(1-\alpha \theta )/2} \right), \\
  \frac{B\lambda^{-t/2}}{\sqrt{n}}
  &= O\left( n^{-(\frac{1-\theta t}{2}-l)} \right);
\end{align*}
For the term \cref{eq:a_Proof_TruncDecomp_C}, we have
\begin{align*}
  \sqrt {\lambda^{-1/\beta} B^{-q}} = O\left( n^{-(ql - \theta/\beta)/2} \right) = O\left( n^{-(1-\theta/\beta)/2} \right).
\end{align*}
Therefore, the desired result follows since all the exponents are negative.
 
\section{Simulation}
\label{sec:appendix-simulation}
\label{sec:simulation}

We give a synthetic Fourier-model experiment illustrating the exact error curve in \cref{thm:Main}.
The model is fully explicit, so the deterministic bias-plus-variance prediction can be compared directly with the conditional generalization error.

We take \(\caX = [0,1)\) with the uniform measure and use the real Fourier basis
\begin{align*}
  e_0(x) &= 1, \quad
  e_{2m-1}(x) = \sqrt{2}\cos(2\pi m x), \quad
  e_{2m}(x) = \sqrt{2}\sin(2\pi m x), \quad m \geq 1.
\end{align*}
with paired power-law eigenvalues
\begin{align*}
  \lambda_0 &= 1, \quad
  \lambda_{2m-1} = \lambda_{2m} = (m+1)^{-\beta}, \quad \beta = 2,
\end{align*}
and coefficients
\begin{align*}
  \theta_j = c \lambda_j^{s/2} j^{-1/2},
  \qquad j \geq 1,
  \qquad s = 1.5,
  \qquad c = 0.7,
\end{align*}
with alternating signs.
The noise has mean zero and variance \(\sigma^2\) with \(\sigma = 0.15\).
We use \(257\) basis functions and \(24\) independent designs for each \(n\).
For each design, we evaluate \( \E\left[\norm{\hat{f}_{\lambda} - \fstar}_{L^2}^2 \middle| X \right] \) exactly by averaging over the noise analytically, and compare it with the theoretical value in \cref{eq:MainResult}.
We report kernel ridge regression and iterated ridge with \(p=2\).

\cref{fig:simulation-curves} shows the full error curve for \(n=128\) and \(n=256\).
For both filters, the empirical conditional error is U-shaped and closely tracks the deterministic bias-plus-variance curve.
The agreement is visibly tighter at \(n=256\), consistent with the \(1+o_{\bbP}(1)\) characterization in \cref{thm:Main}.

\begin{figure}[htbp]
  \centering
  \includegraphics[width=0.95\linewidth]{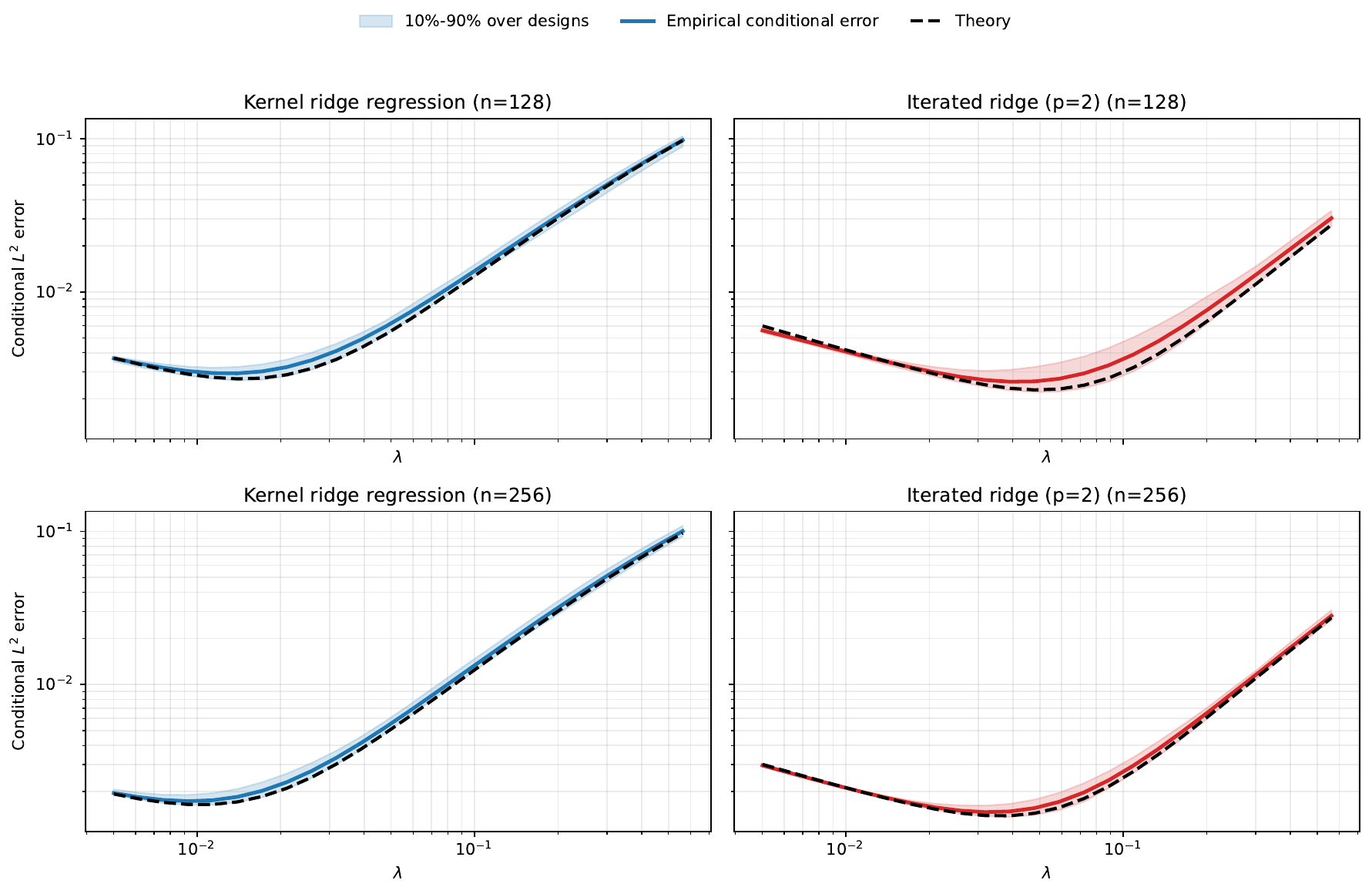}
  \caption{Synthetic error curves on the periodic Fourier model.
  The top row uses \(n=128\) and the bottom row \(n=256\); the left column is kernel ridge regression and the right column is iterated ridge with \(p=2\).
  The solid colored curve is the design-averaged conditional error, with the noise averaged analytically; the shaded region is the \(10\%\) to \(90\%\) band over designs; the dashed black curve represents the theoretical value.}
  \label{fig:simulation-curves}
\end{figure}

We also carried out an exploratory comparison for the non-analytic filters \(\reg^{\mr{cut}}\) and \(\reg^{\mr{clip}}\) introduced in \cref{sec:conclusion}, still using the same deterministic benchmark.
Nevertheless, \cref{fig:nonanalytic-exploration} shows that spectral clipping remains quite close to the same bias-plus-variance curve, while spectral cut-off still shows a visible approximation but in a rougher and less uniform way, especially because of its discontinuous thresholding.

\begin{figure}[htbp]
  \centering
  \includegraphics[width=0.95\linewidth]{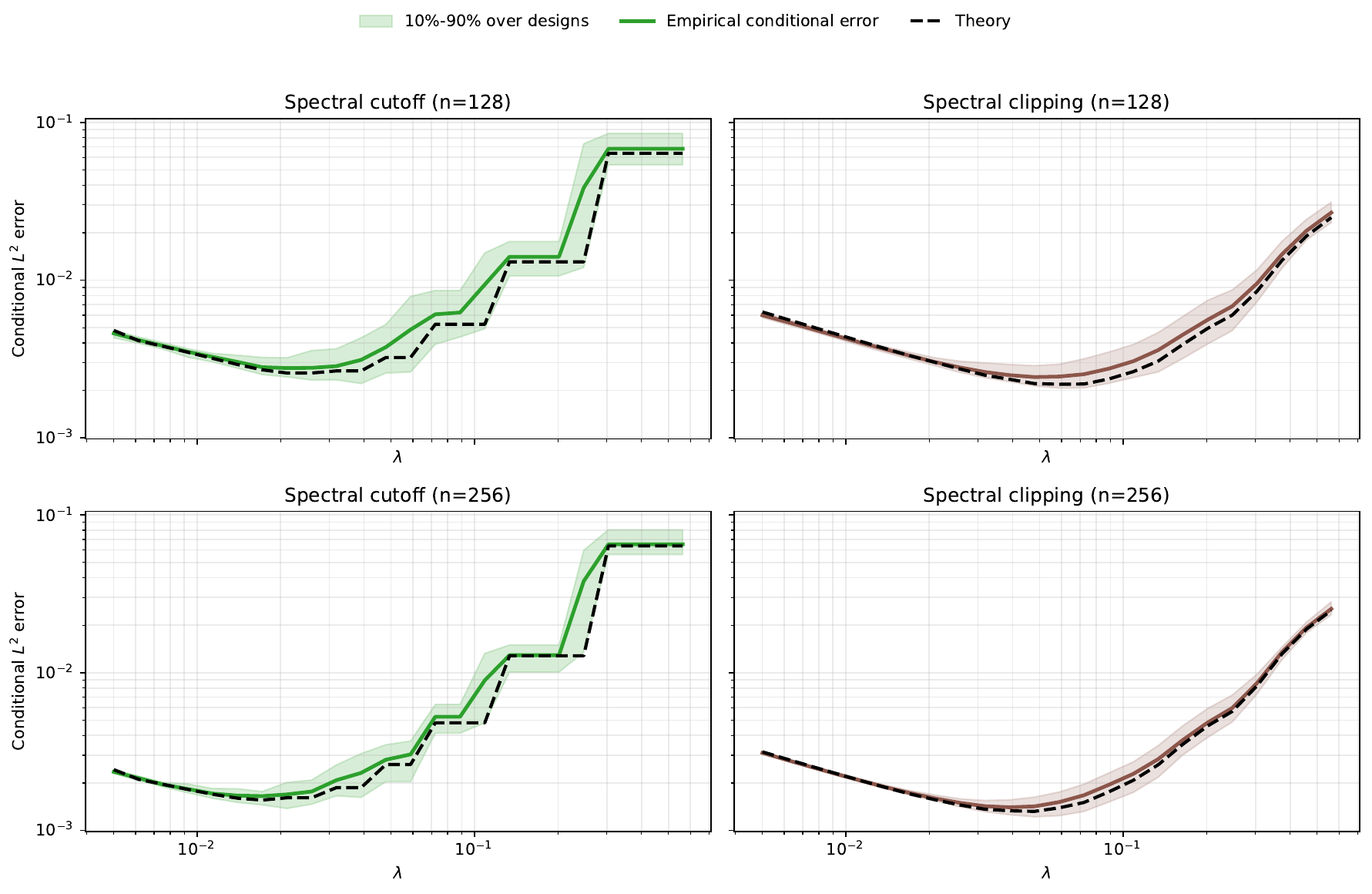}
  \caption{Exploratory error curves for the non-analytic filters \(\reg^{\mr{cut}}\) and \(\reg^{\mr{clip}}\) on the same periodic Fourier model.}
  \label{fig:nonanalytic-exploration}
\end{figure}
 
\section*{Acknowledgement}
This work is supported in part by the National Natural Science Foundation of China (Grant 92370122, Grant 11971257) and the Beijing Natural Science Foundation (Grant Z190001).

\end{document}